\documentclass{article}
\usepackage{iclr2022_workshop,times}

\usepackage{hyperref}
\usepackage{url}

\usepackage{microtype}
\usepackage{graphicx}
\usepackage{subfig}
\usepackage{booktabs} 
\usepackage{tikz}

\usepackage{amsmath}
\usepackage{amssymb}
\usepackage{mathtools}
\usepackage{amsthm}
\usepackage{comment}
\usepackage{adjustbox}

\usepackage[capitalize,noabbrev]{cleveref}

\theoremstyle{plain}
\newtheorem{theorem}{Theorem}[section]
\newtheorem{proposition}[theorem]{Proposition}

\newtheorem{corollary}[theorem]{Corollary}
\theoremstyle{definition}
\newtheorem{definition}[theorem]{Definition}

\theoremstyle{remark}

\title{The PWLR graph representation: \\ A Persistent Weisfeiler-Lehman scheme with Random Walks for graph classification}

\author{Sun Woo Park \thanks{Equal contribution as co-first authors}, \; Yun Young Choi \footnotemark[1], \; Dosang Joe \& Youngho Woo \thanks{ Corresponding author. Email: youngw@nims.re.kr} \\
Division of Industrial Mathematics\\
National Institute for Mathematical Sciences\\
Daejeon, Republic of Korea\\
\texttt{\{spark483, choi930121, dosjoe, youngw\}@nims.re.kr} \\
\And
U Jin Choi \\
Department of Mathematical Sciences \\
KAIST \\
Daejeon, Republic of Korea \\
\texttt{\{ujchoi\}@kaist.ac.kr} \\
}

\iclrfinalcopy 
\begin{document}

\maketitle

\begin{abstract}
This paper presents the Persistent Weisfeiler-Lehman Random walk scheme (abbreviated as PWLR) for graph representations, a novel mathematical framework which produces a collection of explainable low-dimensional representations of graphs with discrete and continuous node features. The proposed scheme effectively incorporates normalized Weisfeiler-Lehman procedure, random walks on graphs, and persistent homology. We thereby integrate three distinct properties of graphs, which are local topological features, node degrees, and global topological invariants, while preserving stability from graph perturbations. This generalizes many variants of Weisfeiler-Lehman procedures, which are primarily used to embed graphs with discrete node labels. Empirical results suggest that these representations can be efficiently utilized to produce comparable results to state-of-the-art techniques in classifying graphs with discrete node labels, and enhanced performances in classifying those with continuous node features. 
\end{abstract}

\section{Introduction}
\label{section: Introduction}

Non-Euclidean data structures are crucial subjects of researches, ranging from interactions among protein structures to social networking systems \cite{SY20, SE20}. Graphs are commonly utilized for modeling the innate properties of a wide class of these non-Euclidean structures \cite{BD21}. A number of strategies producing state-of-the-art results in analyzing the properties of graphs include graph neural networks (GNN) or message passing neural networks (MPNN) \cite{KW17}, graph kernels \cite{VS10}, Weisfeiler-Lehman procedures \cite{SS11, WL68}, random walks \cite{NV20, ZW18, Lo93}, and persistent homological techniques \cite{RBB19, CC20, EH10}. These techniques are often employed to obtain graph representations suitable for graph classifications, which aim to classify innate properties of graphs by detecting their structural differences. The varying topological structures of graphs, however, make the task of obtaining a consistent graph representation demanding.

\textbf{[Related Works]} \; \;
The Weisfeiler-Lehman (WL) isomorphism test measures similarities among graphs with discrete labels by updating the coloring of nodes, each of which represents a depth $k$ unfolding tree \cite{WL68}. Shervashidze et al. implemented the test in the form of graph kernels, producing state-of-the-art results in classifying graph data sets \cite{VS10, SS11}. Various literature focused on generalizing the WL procedure that allows one to represent graphs with continuous node attributes \cite{TG19,BF21}, and incorporate global topological invariants \cite{CK17, RBB19}.

Utilizing stochastic processes on graphs, such as random walks (RW), is an alternative approach to embed graphs to real vector spaces. A random walk on a graph is a random course of travel along the nodes of $G$ characterized by iterative processes of starting from a node and randomly choosing an adjacent node, or itself, to travel to. A number of state-of-the-art approaches include graph kernels using return probabilities of RW \cite{ZW18}, and comparing the number of common random walks \cite{BK05, SB16}.

Persistent homological techniques are known to be effective for computing the global topological invariants of data sets \cite{Ca09, EH10, CR17}. One determines height functions over nodes and edges of $G$ to construct persistence diagrams, which capture the homological properties of $G$. The encapsulated properties depend on which features of graphs the predetermined height function utilizes. These features include node labels \cite{RBB19}, and spectral decompositions of the adjacency matrix \cite{CC20}.

\textbf{[Motivation]} \; The motivation for this project originates from instability and non-optimal dimensionality of representations obtained from pre-existing approaches. Variants of GNNs do not necessarily guarantee representation stability with respect to graph perturbations, where a pair of graphs with similar topological structures may be embedded to a pair of vectors whose distance between them may be arbitrarily large \cite{XH19}. The WL procedure also substantially increases the dimensionality of representations as the number of iterations of the procedure increase \cite{SS11}. While graph kernels effectively bound these dimensions, they produce representations that depend on the choice of training data sets \cite{HK03, VS10}. 

To address these limitations, we propose the Persistent Weisfeiler-Lehman Random Walk embedding framework (PWLR), a novel graph embedding formalism which produces a collection of low-dimensional representations of graphs while preserving stability from graph perturbations. Inspired from WL \cite{SS11} and PWL \cite{RBB19} procedures, the PWLR framework captures local topological features, node degrees, and global topological invariants of graphs by effectively incorporating the normalized WL procedure, random walks, and persistent homological approaches. Experimental results suggest that the proposed algorithm produces comparable results to state-of-the-art techniques in classifying graphs with discrete node labels, and enhanced performances in classifying those with continuous node attributes or edge weights. 

\textbf{[Contributions]} \; \; 
The novelty of the PWLR embedding framework can be summarized as follows.

$\bullet$ A mathematical framework incorporating \textbf{three distinct topological properties} (Theorem \ref{theorem: three_properties})

$\bullet$ Effective \textbf{low-dimensional} representations for classifying graphs with discrete and \textbf{continuous features}. (Table \ref{tab:results}, \ref{tab:dimension})

$\bullet$ \textbf{Stability} of representations with respect to graph perturbations (Theorem \ref{theorem:stability})

\section{Persistent Weisfeiler-Lehman Random Walk Graph Representation}
\label{section: PWLR}


\textbf{[Algorithm]} \; \; 
The PWLR scheme obtains a vector representation of a finite graph $G$ as follows: \footnote{Appendix provided at: } \footnote{Github repository: https://github.com/spark483/The-PWLR-graph-representation}
\begin{equation} \label{equation:PWLR_embedding_scheme}
    \varphi_{\text{PWLR}}(G) := \varphi \left( \left( M_G^{k_1} X \right)^T M_G^{k_2} \right).
\end{equation}
We note that $k_1, k_2$ are positive integers, $X$ is a $|V| \times l$ matrix consisting of a concatenation of node labels of $G$, $M_G$ is the normalized weighted adjacency matrix of dimension $|V| \times |V|$ given by
\begin{equation} \label{equation:transition_matrix}
    M_G := (D+I)^{-1}(A+I),
\end{equation}
$(\cdot)^T$ is the transpose of a matrix, and $\varphi$ is the Euclidean embedding obtained from persistent homological features utilizing the updated node labels $(M_G^{k_1} X)^T M_G^{k_2}$. The matrix $M_G$ is the normalized adjacency convolutional operator used in message passing neural networks (MPNN) \cite{KW17, CJM20, VC18}. Theorem \ref{theorem: three_properties} provides a list of correspondences among each component of (\ref{equation:PWLR_embedding_scheme}) and the utilized algorithms for analyzing finite graphs, the proof of which is in Appendix \ref{appendix: theorem: three_properties}.
\begin{theorem}
\label{theorem: three_properties}
Given a finite undirected graph $G = (V,E)$ without self-loops, the PWLR procedure incorporates three disjoint topological properties of $G$. 

$\bullet$ The component $M_G^{k_1} \times (\cdot)$ (\textbf{WL}) incorporates local topological properties by representing depth $k_1$ unfolding trees with fixed vertices. 

$\bullet$ The component $(\cdot) \times M_G^{k_2}$ (\textbf{R}) incorporates node degrees with local topological properties. 

$\bullet$ Lastly, the component $\varphi(\cdot)$ (\textbf{P}) incorporates homological invariants of $G$, i.e. the connected components and cycles.
\end{theorem}
The parameters $k_1$ and $k_2$, each corresponding to the number of iterations of the WL or RW procedure, provide a collection of adaptable graph representations suitable for classification tasks. They determine the extent of incorporating node attributes and node degrees in constructing graph representations. The map $\varphi$ incorporates global topological characteristics by constructing persistence diagrams from preset height functions, as will be shown in (\ref{eq:first_representation}, \ref{eq:second_representation}).

Each property mentioned in Theorem \ref{theorem: three_properties} is known to be crucial for distinguishing isomorphism classes of graphs. The WL procedure and its variants are well-known for their state-of-the-art results in classifying graphs \cite{SS11, TG19}. Geerts, Mazowiecki, and Per\'{e}z proved that incorporating node degrees enhances conventional message passing neural networks (MPNN) in distinguishing nodes based on their attributes \cite{GMP21}. In addition, persistent homological algorithms are known to be effective for enhancing the accuracy of graph classification algorithms \cite{RBB19, CC20}. Interested readers may refer to Appendix \ref{appendix: mathematical_background} for expositions on random walks and persistent homology. A pseudo-code and a diagram summarizing the PWLR embedding scheme are outlined in Algorithm \ref{alg:PWLR} and Figure \ref{figure:blueprint}. Explicit computations on how the PWLR algorithm constructs graph representations can be found in Appendix \ref{appendix: examples}.


\begin{algorithm}[tb]
   \caption{PWLR Embedding Framework}
   \label{alg:PWLR}
\begin{algorithmic}
    \scriptsize
    \STATE {\bfseries Input:} Graph $G = (V,E)$, iterations $(k_1, k_2)$, $p=1, M_G := (D+I)^{-1}(A+I)$, $X := \{L(v_i)\}_{v_i \in V}$.
        \STATE Initialize $\overline{D}_e := \{\overline{d}_e := (\overline{d}_v,\overline{d}_w) \; | \; e = (v,w) \in E\}$. \; ($\overline{d}_v$ is the unweighted degree at $v$.)
        \STATE Set $\varphi_{\text{H0}}, \varphi_{\text{H1}} := \emptyset$, \; $\varphi_{\text{H0,Opt}}, \varphi_{\text{H1,Opt}} := [0,\cdots, 0]$ (of length $|\overline{D}|$).
    \STATE {\bfseries Normalized WL:} Update $X \gets M_G^{k_1} X$
    \STATE {\bfseries Random Walk:} Update $X \gets X^T M_G^{k_2}$
    \STATE {\bfseries Persistent homological embedding:} Define $h_E((v_1,v_2)) := \|X(v_1)-X(v_2)\|_p$
        \STATE Obtain a set $\tilde{E}$ by sorting the set $E(G)$ using $h_E$: $\tilde{E} := \{e_i \in E(G) \; | \; h_E(e_i) \leq h_E(e_j) \text{ if } i \leq j \}$.
        \FOR{$i \in \{1,\cdots,|\tilde{E}|\}$}
            \STATE Define $E^{[i]} := \{e_j \in \tilde{E} \; | \; j \leq i\}$, \; Initialize $G^{[i-1]} := (V,E^{[i-1]})$, $G^{[i]} := (V,E^{[i]})$.
            \STATE Pick $e_i \in E^{[i]} \setminus E^{[i-1]}$ and compute $\overline{d}_{e_i} := (\overline{d}_{v_1^i}, \overline{d}_{v_2^i})$.
            \IF {$\# \text{ Comp. } G^{[i]} - \# \text{ Comp. } G^{[i-1]} > 0$}
                \STATE $\varphi_{\text{H0}} \gets \text{Concat}(\varphi_{\text{H0}},[h_E(e_i)+1])$, \; $\varphi_{\text{H0,Opt}}(\overline{d}_{e_i}) \gets \varphi_{\text{H0,Opt}}(\overline{d}_{e_i}) + (h_E(e_i)+1)$
            \ENDIF
            \IF {$\# \text{ Cycle } G^{[i]} - \# \text{ Cycle } G^{[i-1]}> 0$}
                \STATE $\varphi_{\text{H1}} \gets \text{Concat}(\varphi_{\text{H1}},[h_E(e_i)]+1)$, \; $\varphi_{\text{H1,Opt}}(\overline{d}_{e_i}) \gets \varphi_{\text{H1,Opt}}(\overline{d}_{e_i}) + (h_E(e_i)+1)$
            \ENDIF
        \ENDFOR
    \STATE {\bfseries Return:} $\varphi_{\text{H0}}, \varphi_{\text{H1}}, \varphi_{\text{H0,Opt}}, \varphi_{\text{H1,Opt}}$.
\end{algorithmic}
\end{algorithm}

\begin{figure*}[ht]
\vskip 0.2in
\begin{center}
\centerline{\includegraphics[width=130mm]{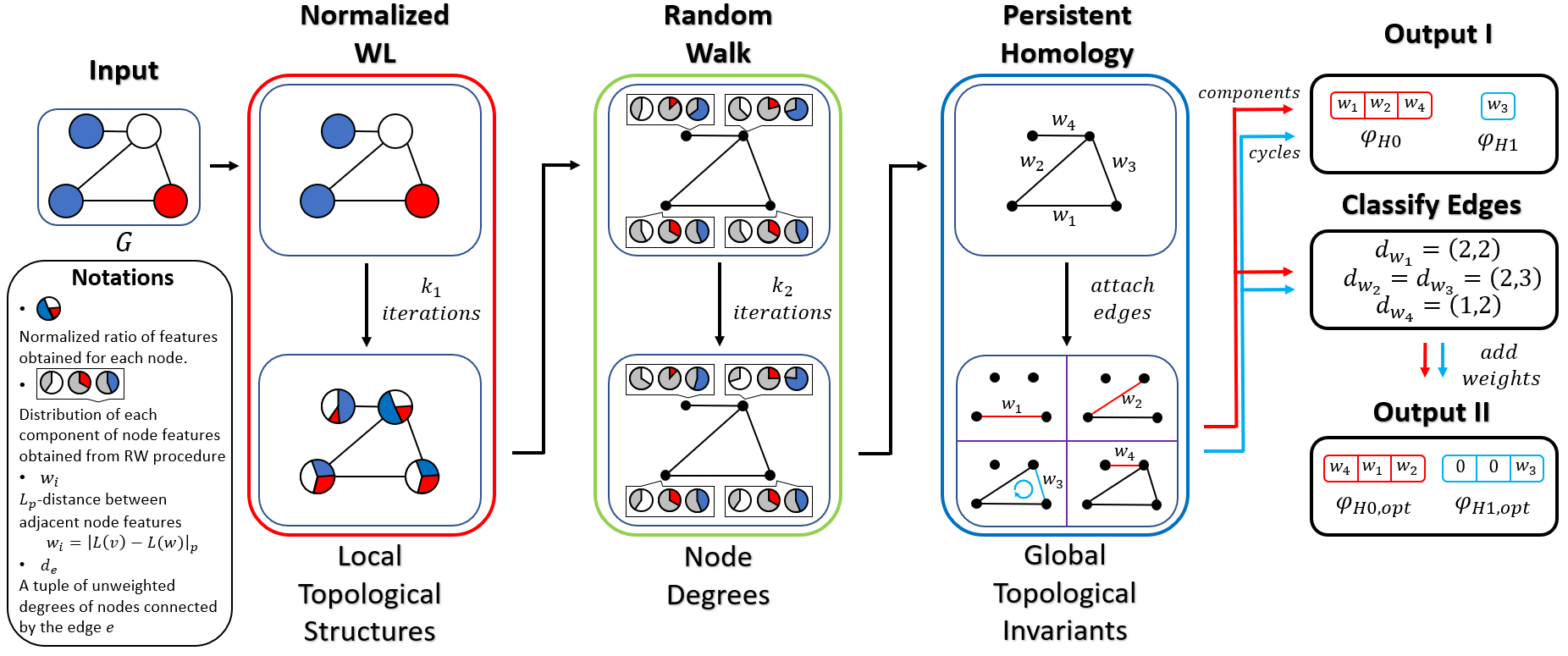}}
\caption{A diagram outlining the architecture of the PWLR embedding scheme}
\label{figure:blueprint}
\end{center}
\vskip -0.2in
\end{figure*}

\textbf{[Normalized WL procedure]} \; \; 
We analyze the correspondence between the operator $M_G^{k_1} \times (\cdot)$ and the node labels obtained from the normalized WL procedure. We first define the normalized WL procedure, obtained from normalizing a node label updated from a WL procedure by its length.
\begin{definition}[Normalized WL procedure]
Let $G := (V,E)$ be a graph with node attributes $L:V \to \mathbb{R}_{>0}^l$. Given a node $v \in V$, denote by $N(v)$ the set of nodes which are adjacent to $v$, including $v$ itself. Let $\|\cdot\|_1$ be the $l_1$-norm of a node label, and $\tilde{a}_{v,w}$ be the weight on edges $a_{v,w}$ from node $v$ to $w$, with $\tilde{a}_{v,v}=1$. The normalized WL procedure updates the attribute of $v$ by
\begin{equation}
    L(v) \mapsto \frac{\sum_{w \in N(v)} \tilde{a}_{v,w} L(w)}{\sum_{w \in N(v)} \tilde{a}_{v,w} \|L(w)\|_1}.
\end{equation}
\end{definition}
If we further assume that the $l_1$-norms of all node attributes are equal to $1$, then we immediately obtain the following correspondence stated in Proposition \ref{proposition: normalized_wl}. The proposition can be applied to graphs with discrete node labels, where each node label is embedded to real vector spaces using the one-hot encoder, see Appendix \ref{appendix: examples} for instance.
\begin{proposition} \label{proposition: normalized_wl}
Suppose that for every node $v \in V(G)$, the $l_1$ norm of node labels is equal to $1$, i.e. $\|L(v)\|_1 = 1$. Then the right multiplication $M_G^{k_1} \times (\cdot)$ is equivalent to $k_1$ iterations of normalized WL procedure applied to $G$.
\end{proposition}
Using the above proposition, the over-smoothening phenomenon of WL procedure can be reinterpreted by using the eigenvectors of $M_G$. Because $M_G$ is a stochastic matrix (a square matrix whose sum of entries is equal to $1$ for each row), a vector $\nu$ whose entries are all equal to $1$ is the right eigenvector of $M_G$ with eigenvalue $1$. Denote by $\pi_G$ the left eigenvector of $M_G$ with eigenvalue $1$, whose entries will be computed in Theorem \ref{appendix:theorem:stationary_distribution}. The limiting behavior of $M_G^k X$ for sufficiently large $k$ can be obtained from the Perron-Frobenius theorem (See Theorem Appendix \ref{theorem:perron_frobenius}):
\begin{equation}
    \lim_{k \to \infty} M_G^k X = 
    \nu \pi_G^T X
\end{equation}
Note that $\nu \pi_G^T X$ is a matrix whose entries for each column are all identical. We can hence reinterpret the over-smoothening phenomenon of the WL procedure as the limiting behavior of the operator $M_G$.

\textbf{[Random Walks]} \; \; 
Let $P \in \mathbb{R}^n$ be a probability distribution over the set of nodes of $G$. The random walk on $G$ updates $P$ by multiplying its transpose with $M_G$ to the right, i.e. $P \mapsto P^T M_G$, where $P^T$ is the transpose of $P$. We generalize the construction above by substituting the probability distribution $P$ with the matrix of concatenated node attributes $X$. 
\begin{definition}
Let $G := (V,E)$ be a graph with a matrix of concatenated node attributes $X$. Fix a positive integer $k_2$. We say that the node attributes are updated from $k_2$ iterations of random walks (RW) over $G$ if the matrix $X$ is updated to $X^T M_G^{k_2}$.
\end{definition}
The core difference between WL and RW procedure lies in the difference between left and right eigenvectors of $M_G$. To elaborate, the left eigenspace of $M_G$ with eigenvalue $1$ is spanned by the probability distribution of node degrees. 
\begin{theorem}[(Section 1, \cite{Lo93})]
\label{appendix:theorem:stationary_distribution}
Let $G := (V,E)$ be a finite graph. The left eigenspace of $M_G$ with eigenvalue $1$ is spanned by
\begin{equation} \label{equation:stationary_distribution}
    \pi_G := \biggl[ \frac{d_v}{\sum_{w \in V(G)} d_w }\biggr]_{v \in V(G)}.
\end{equation}
\end{theorem}
Denote by $\nu$ the column vector whose entries are all equal to $1$. By the Perron-Frobenius Theorem (Theorem Appendix \ref{theorem:perron_frobenius})
\begin{equation}
    \lim_{k \to \infty} X^T M_G^k = X^T \nu \pi_G^T,
\end{equation}
where each row of $X^T \nu \pi_G^T$ lies in the span of $\pi_G^T$. Thus, iterations of RW procedures incorporate information on node degrees with the given node features.

\textbf{[Persistent Homology]} \; \; 
Given $k_1$ iterations of WL procedure and $k_2$ iterations of RW procedure, we update the matrix of node labels $X$ to $(M_G^{k_1}X)^T M_G^{k_2}$. Given a node $v \in V$, we use the abbreviation $X^{[k_1,k_2]}(v)$ to denote the features of the node $v$ obtained from the matrix $(M_G^{k_1}X)^T M_G^{k_2}$. 

In the spirit of persistent WL (PWL) procedure \cite{RBB19}, we characterize the global topological invariants of $G$ by constructing a sequence of nested subgraphs induced from the updated matrix $(M_G^{k_1}X)^T M_G^{k_2}$. We define the height functions $h_V:V \to \mathbb{R}$ and $h_E:E \to \mathbb{R}$ as
\begin{align} \label{equation:height_function}
\begin{split}
    h_V(v) &= 0 \text{ for all } v \in V \\
    h_E(v_1, v_2) &= \|X^{[k_1,k_2]}(v_1) - X^{[k_1,k_2]}(v_2)\|_p
\end{split}
\end{align}
where $\|\cdot\|_p$ is the $l_p$-norm over $\mathbb{R}^{l}$. Let $\tilde{E}$ be the set $E(G)$ that is sorted using the function $h_E$:
\begin{equation} \label{equation:sorted_edge}
    \tilde{E} := \{e_i \in E(G) \; | \; h_E(e_i) \leq h_E(e_j) \text{ if } i \leq j \}
\end{equation}
The set of edges $\tilde{E}$ is sorted in a manner that their heights are in an increasing order. Using the $i$-th edge $e_i$ of the sorted list $\tilde{E}$, we define the subgraphs $G^{[i]}$ as
\begin{align}
\begin{split}
    G^{[i]} &:= (V,E^{[i]} := \{e_j \in \tilde{E} \; | \; j \leq i \} )
\end{split}
\end{align}
We now obtain a sequence of nested subgraphs of $G$:
\begin{equation}
    G^{[0]} \subset G^{[1]} \subset G^{[2]} \subset \cdots \subset G^{|E|} = G.
\end{equation}
The nodes of all $G^{[i]}$'s are fixed, whereas the edges are added in increasing order of their weights.

\begin{table*}[]
    \caption{A part of classification results obtained from data sets with discrete node labels and continuous node attributes. Cells notated as N/A indicate graph classification schemes which do not report classification results on the given graph data set or may have limitations in processing the node features and edge weights present in the given graph data set. The entries for PROTEINS, BZR, and COX-2 data sets report the highest classification results obtained from using discrete node labels, or from using both discrete and continuous attributes if possible. All classification results other than the proposed method are obtained from pre-existing publications.}
    \label{tab:results}
    \vskip 0.15in
    \begin{center}
    \begin{scriptsize}
    \begin{sc}
    \begin{adjustbox}{width=\textwidth}
    \begin{tabular}{ccccccccc}
            \hline
             & MUTAG & PTC-FR & NCI1 & PROTEINS & BZR & BZR-MD & COX2 & COX2-MD \\
            \hline
            \hline
            WL & 88.72$\pm$1.11 & 67.64$\pm$0.74 & 85.58$\pm$0.15 & 76.11$\pm$0.64 & N/A & N/A & N/A & N/A \\
            PWL-H0 & 86.10$\pm$1.37 & 67.30$\pm$1.50 & 85.34$\pm$0.14 & 75.31$\pm$0.73 & N/A & N/A & N/A & N/A \\
            PWL-H1 & 90.51$\pm$1.34 & 67.15$\pm$1.09 & 85.46$\pm$0.16 & 75.27$\pm$0.38 & N/A & N/A & N/A & N/A \\
            \hline
            WWL & 87.27$\pm$1.50 & N/A & 85.75$\pm$0.25 & 77.91$\pm$0.80 & 84.42$\pm$2.03 & 69.76$\pm$0.94 & 78.29$\pm$0.47 & \textbf{76.33}$\pm$\textbf{1.02} \\
            RetGK-I,II & 90.30$\pm$1.10 & 67.80$\pm$1.10 & 84.50$\pm$0.20 & \textbf{78.00}$\pm$\textbf{0.30} & 87.10$\pm$0.70 & 62.77$\pm$1.69 & \text{81.40}$\pm$\text{0.60} & 59.47$\pm$1.66 \\
            HGK-WL & N/A & N/A & N/A & 76.70$\pm$0.41 & 78.59$\pm$0.63 & 68.94$\pm$0.65 & 78.13$\pm$0.45 & 74.61$\pm$1.74 \\
            FGW & 88.42$\pm$5.67 & N/A & \textbf{86.42}$\pm$\textbf{1.63} & 74.55$\pm$2.74 & 85.12$\pm$4.15 & N/A & 77.23$\pm$4.86 & N/A \\
            FC & 87.31$\pm$0.66 & N/A & N/A & 74.54$\pm$0.48 & 85.61$\pm$0.59 & 75.61$\pm$1.13 & 81.01$\pm$0.88 & 73.41$\pm$0.79 \\
            \hline
            \hline
            GCKN-subtree & 91.60$\pm$6.70 & 68.40$\pm$7.40 & 82.00$\pm$1.20 & 77.60$\pm$0.40 & 86.40$\pm$0.50 & N/A & \textbf{81.70}$\pm$\textbf{0.70} & N/A \\
            DM & N/A & 68.39$\pm$3.57 & 83.07$\pm$1.07 & 76.19$\pm$2.91 & N/A & 73.55$\pm$5.76 & N/A & 72.28$\pm$9.37 \\
            Perslay & 89.80$\pm$0.90 & N/A & 73.50$\pm$0.30 & 74.80$\pm$0.30 & N/A & N/A & 80.90$\pm$1.00 & N/A \\
            \hline
            \hline
            \textbf{PWLR-H0} & 89.52$\pm$0.90 & \textbf{69.12}$\pm$\textbf{1.16} & 75.66$\pm$0.30 & 74.62$\pm$0.62 & \textbf{89.39}$\pm$\textbf{0.74} & \text{75.08}$\pm$\text{1.32} & 80.61$\pm$1.11 & 71.36$\pm$1.46 \\
            \textbf{PWLR-H1} & \textbf{91.97}$\pm$\textbf{0.92} & 67.04$\pm$0.67 & 74.50$\pm$0.42 & 73.02$\pm$0.50 & \textbf{89.29}$\pm$\textbf{0.66} & \textbf{77.75}$\pm$\textbf{0.98} & 81.02$\pm$0.55 & 69.96$\pm$1.26 \\
            \textbf{PWLR-H0+H1} & 89.47$\pm$1.51 & \textbf{68.75}$\pm$\textbf{1.39} & 77.15$\pm$0.23 & 73.94$\pm$0.58 & \textbf{88.95}$\pm$\textbf{0.90} & \textbf{76.44}$\pm$\textbf{1.37} & 80.88$\pm$0.54 & 71.36$\pm$1.46 \\
            \hline
            \textbf{PWLR-OPT-H0} & 89.73$\pm$1.01 & 64.02$\pm$0.99 & 75.99$\pm$0.23 & 74.10$\pm$0.59 & \textbf{89.46}$\pm$\textbf{0.55} & \text{72.47}$\pm$\text{1.44} & 79.94$\pm$0.58 & 72.10$\pm$2.13 \\
            \textbf{PWLR-OPT-H1} & \textbf{91.91}$\pm$\textbf{1.61} & 66.19$\pm$1.23 & 72.76$\pm$0.33 & 73.22$\pm$0.39 & \textbf{88.93}$\pm$\textbf{0.79} & \text{74.18}$\pm$\text{1.18} & 80.22$\pm$0.84 & 70.51$\pm$1.86 \\
            \textbf{PWLR-OPT-H0+H1} & 89.17$\pm$0.84 & 65.73$\pm$1.16 & 79.05$\pm$0.40 & 74.46$\pm$0.38 & \textbf{89.32}$\pm$\textbf{0.83} & \text{75.55}$\pm$\text{1.08} & 80.97$\pm$1.15 & 72.97$\pm$1.00 \\
            \hline
    \end{tabular}
    \end{adjustbox}
    \end{sc}
    \end{scriptsize}
    \end{center}
    \vskip -0.1in
\end{table*}

$\circ$ \textbf{Vector Representation:} For each subgraph $G^{[i]}$, we compute the number of its connected components and cycles, corresponding to its $0$-th and $1$-st homology groups. The ranks of these groups are known as \textit{Betti numbers}, denoted respectively as $\beta_0(G^{[i]})$ and $\beta_1(G^{[i]})$. As more edges are added, the 0-th Betti numbers of subgraphs decrease, whereas the 1-st Betti numbers increase. Given an edge $e_i \in \tilde{E}$ that is included in $G^{[i]}$ but not in $G^{[i-1]}$, the variations in homological invariants can be computed as follows:
\begin{align}
    h_0^i &:= \beta_0(G^{[i-1]}) - \beta_0(G^{[i]}) \\
    h_1^i &:= \beta_1(G^{[i]}) - \beta_1(G^{[i-1]})
\end{align}
Euler's characteristic formula (Appendix \ref{theorem: euler_characteristic}) implies that whenever an edge $e \in E^{[i]}$ is newly added, either $\beta_0$ decreases by $1$, or $\beta_1$ increases by $1$. Given a connected graph $G$, there are $|V|-1$ heights on edges which record decrements of $\beta_0$, and $|E|-|V|+1$ heights on edges which record increments of $\beta_1$. The sorted lists of such heights induce the following representations of graphs:
\begin{align} \label{eq:first_representation}
\begin{split}
    \varphi_{H_0}^{[k_1,k_2]} &:= [h_E(e_i) + \tau \; | \; h_0^i > 0]_{e_i \in \tilde{E}} \\
    \varphi_{H_1}^{[k_1,k_2]} &:= [h_E(e_i) + \tau \; | \; h_1^i > 0]_{e_i \in \tilde{E}}
\end{split}    
\end{align}
The bias term $\tau$ (usually equal to $1$) distinguishes cases where the height levels $h_E(e_i)$ are equal to $0$ from those where the edges with heights $0$ do not occur in $G$.

$\circ$ \textbf{Reduced Dimensions:} To further reduce the embedded dimensions, we may record these heights using the following procedure. Any edge $e = (v_1,v_2)$ can be represented as a tuple of unweighted degrees of two nodes $\overline{d}_e := (\overline{d}_{v_1},\overline{d}_{v_2})$. We denote by $\overline{D}_E$ the set of tuples of unweighted degrees of two nodes connected by an edge:
\begin{equation} \label{equation:set_unweighted_degrees}
    \overline{D}_E := \{\overline{d}_e := (\overline{d}_v,\overline{d}_w) \; | \; e = (v,w) \in E\}
\end{equation}
Using the set of tuples, we represent the graphs as $|\overline{D}_E|$-dimensional real vectors by taking constrained summations of sorted heights over the set of edges, based on their associated tuples of unweighted degrees $\overline{d_e}$. Given a fixed tuple $\overline{d}$ of unweighted degrees, the $\overline{d}$-components of the representations $\varphi_{H_0,Opt}^{[k_1,k_2]}$ and $\varphi_{H_1,Opt}^{[k_1,k_2]}$ are sums of heights on edges $e_i$ such that the values $h_0^i$ (or $h_1^i$, respectively) are positive, and that the tuples of unweighted degrees $\overline{d}_{e_i}$ are equal to $\overline{d}$. The explicit definition of $\varphi_{H_0,Opt}^{[k_1,k_2]}$ and $\varphi_{H_1,Opt}^{[k_1,k_2]}$ can be thus summarized as shown in the following equation:
\begin{align} \label{eq:second_representation}
    \begin{split}
        \varphi_{H_0,Opt}^{[k_1,k_2]}(\overline{d}) &:= \sum_{\substack{e_i \in \tilde{E} \text{ such that } \\ h_0^i > 0, \; \overline{d}_{e_i} = \overline{d}}} (h_E(e_i) + \tau) \\
        \varphi_{H_1,Opt}^{[k_1,k_2]}(\overline{d}) &:= \sum_{\substack{e_i \in \tilde{E} \text{ such that } \\ h_1^i > 0, \; \overline{d}_{e_i} = \overline{d}}} (h_E(e_i) + \tau)
    \end{split}
\end{align}

\textbf{[Representation Stability]} \; \; 
It is imperative to verify whether the representations from (\ref{eq:first_representation}, \ref{eq:second_representation}) preserve stability with respect to graph perturbations. Carri\`{e}re et al. verified that the heat kernel signature preserves representation stability with respect to graph perturbations \cite{CC20, CS16}. As for the PWLR scheme, the incorporation of three algorithms allows us to numerically compute the upper bound of the distance between two representations of graphs. We leave the proof of the stability theorem in Appendix \ref{appendix: theorem: stability}, which uses geometric ergodicity and perturbation theory of random walks over finite graphs \cite{Sc68, Lo93}.

\begin{table*}[]
    \caption{A part of classification results obtained from data sets with discrete node labels and continuous node attributes. Cells notated as N/A indicate graph classification schemes which do not report classification results on the given graph data set or may have limitations in processing the node features and edge weights present in the given graph data set. The entries for PROTEINS, BZR, and COX-2 data sets report the highest classification results obtained from using discrete node labels, or from using both discrete and continuous attributes if possible. All classification results other than the proposed method are obtained from pre-existing publications.}
    \label{tab:results}
    \vskip 0.15in
    \begin{center}
    \begin{scriptsize}
    \begin{sc}
    \begin{adjustbox}{width=\textwidth}
    \begin{tabular}{ccccccccc}
            \hline
             & MUTAG & PTC-FR & NCI1 & PROTEINS & BZR & BZR-MD & COX2 & COX2-MD \\
            \hline
            \hline
            WL & 88.72$\pm$1.11 & 67.64$\pm$0.74 & \textbf{85.58}$\pm$\textbf{0.15} & 76.11$\pm$0.64 & N/A & N/A & N/A & N/A \\
            PWL-H0 & 86.10$\pm$1.37 & 67.30$\pm$1.50 & 85.34$\pm$0.14 & 75.31$\pm$0.73 & N/A & N/A & N/A & N/A \\
            PWL-H1 & 90.51$\pm$1.34 & 67.15$\pm$1.09 & 85.46$\pm$0.16 & 75.27$\pm$0.38 & N/A & N/A & N/A & N/A \\
            \hline
            WWL & 87.27$\pm$1.50 & N/A & 85.75$\pm$0.25 & 77.91$\pm$0.80 & 84.42$\pm$2.03 & 69.76$\pm$0.94 & 78.29$\pm$0.47 & \textbf{76.33}$\pm$\textbf{1.02} \\
            RetGK-I,II & 90.30$\pm$1.10 & 67.80$\pm$1.10 & 84.50$\pm$0.20 & \textbf{78.00}$\pm$\textbf{0.30} & 87.10$\pm$0.70 & 62.77$\pm$1.69 & \text{81.40}$\pm$\text{0.60} & 59.47$\pm$1.66 \\
            \hline
            GCKN-subtree & 91.60$\pm$6.70 & 68.40$\pm$7.40 & 82.00$\pm$1.20 & 77.60$\pm$0.40 & 86.40$\pm$0.50 & N/A & \textbf{81.70}$\pm$\textbf{0.70} & N/A \\
            Perslay & 89.80$\pm$0.90 & N/A & 73.50$\pm$0.30 & 74.80$\pm$0.30 & N/A & N/A & 80.90$\pm$1.00 & N/A \\
            \hline
            \hline
            \textbf{PWLR} & \textbf{91.97}$\pm$\textbf{0.92} & \textbf{69.12}$\pm$\textbf{1.16} & 77.15$\pm$0.23 & 74.62$\pm$0.62 & \textbf{89.39}$\pm$\textbf{0.74} & \textbf{77.75}$\pm$\textbf{0.98} & 81.02$\pm$0.55 & 71.36$\pm$1.46 \\
            \textbf{PWLR-OPT} & \textbf{91.91}$\pm$\textbf{1.61} & 66.19$\pm$1.23 & 79.05$\pm$0.40 & 74.46$\pm$0.38 & \textbf{89.46}$\pm$\textbf{0.55} & 75.55$\pm$1.08 & 80.97$\pm$1.15 & 72.97$\pm$1.00 \\
            \hline
    \end{tabular}
    \end{adjustbox}
    \end{sc}
    \end{scriptsize}
    \end{center}
    \vskip -0.1in
\end{table*}

\begin{theorem} [Stability for PWLR graph representations]
\label{theorem:stability}
Let $G,G'$ be two connected graphs with the same number of nodes. Let $\epsilon$ be defined as $\epsilon := M_G - M_{G'}$. Denote by $0 < \mu_{2,G}, \mu_{2,G'} < 1$ the second largest eigenvalues of $M_G$ and $M_{G'}$. Then under certain conditions (see Theorem Appendix \ref{theorem:appendix_stability}), there exists a fixed constant $C > 0$ such that
\begin{align*}
\begin{split}
\|\varphi_{H_i}^{[k_1,k_2]}(G) - \varphi_{H_i}^{[k_1',k_2']}(G')\|_1 < C (\mu_{2,G}^{k_2} + \mu_{2,G'}^{k_2'} + \|\epsilon\|_1)
\end{split}
\end{align*}
\end{theorem}

A generalization of Theorem \ref{theorem:stability} for both representations from (\ref{eq:first_representation}, \ref{eq:second_representation}) can be found in Theorem \ref{theorem:appendix_stability} and Corollary \ref{corollary:appendix_stability} in Appendix \ref{appendix: technical_proofs}. The differences between two representations are numerically controlled by the second largest eigenvalues of $M_G$ and $M_{G'}$ and the perturbation matrix $\epsilon$. The eigenvalues are strictly less than $1$ because both matrices are stochastic. Hence, for sufficiently large $k_2$, the $l_1$-distance between two vectors is controlled by the norm of $\epsilon$. Theorem Appendix \ref{theorem:stationary_distribution} further shows that the updated node features converge to the probability distribution of node degrees. Hence, the PWLR scheme quantifies the extent of incorporating node degrees to graph representations.

\textbf{[Time Complexity]} \; \; 
The total time complexity for embedding graphs with $l$-dimensional node attributes using the PWLR procedure up to $k_1$-iterations of WL kernel and $k_2$-iterations of RW is $\mathcal{O}(k_1 \times k_2 \times m \times (l + \log m))$. We refer to Appendix \ref{subsection:time_complexity} for further details on computing the time complexity of the PWLR algorithm.

\begin{table*}
    \caption{Dimensions of some graph representations processed in the Random Forest Classifier for classifying graphs. Cells notated as ''-`` indicate graph data sets which do not have the prescribed node features. For the first four data sets, the dimensions of representations constructed from WL and PWL procedures are obtained by processing the discrete node labels of graphs, ignoring any continuous node features or attributes on edges. The variable $h$ denotes the number of WL iteration procedures used for obtaining the representations. The asterix ''$*$`` indicates that the obtained graph representations are subject to changes based on the choice of training data sets.}
    \label{tab:dimension}
    \vskip 0.15in
    \begin{center}
    \begin{scriptsize}
    \begin{sc}
    \begin{adjustbox}{width=\textwidth}
    \begin{tabular}{ccccccccc}
    \hline
        data sets & MUTAG & PTC-FR & NCI1 & PROTEINS & BZR & BZR-MD & COX2 & COX2-MD \\
        \hline
        \hline
        Average \# nodes & 17.93 & 14.56 & 29.87 & 39.06 & 35.75 & 21.30 & 41.22 & 26.28 \\
        Average \# edges & 19.79 & 15.00 & 32.30 & 72.82 & 38.36 & 225.06 & 43.45 & 335.12 \\
        Discrete Labels & 7 & 19 & 22 & 3 & 10 & 8 & 8 & 7 \\
        Continuous Features & - & - & - & 29 & 3 & - & 3 & - \\
        Edge Attributes & - & - & - & - & - & 1 & - & 1 \\
        \# Graphs & 188 & 351 & 4110 & 1113 & 405 & 306 & 467 & 303 \\
        \hline
        \hline
        Graph Kernels (10-fold)* & 169* & 315* & 3699* & 996* & 364* & 275* & 420* & 272* \\
        \hline
        WL, PWL-H0 ($h=1$) & 40 & 148 & 288 & 299 & N/A & N/A & N/A & N/A \\
        WL, PWL-H0 ($h=10$) & 15,969 & 27,139 & 530,723 & 329,035 & N/A & N/A & N/A & N/A \\
        WL, PWL-H0 ($h=20$) & 41,825 & 63,404 & $\geq 10^6$ & 721,222 & N/A & N/A & N/A & N/A \\
        \hline
        \hline
        \textbf{PWLR-H0} (Any $k_1$,$k_2$) & \textbf{28} & \textbf{64} & \textbf{111} & 620 & \textbf{54} & \textbf{56} & \textbf{33} & \textbf{36} \\
        \textbf{PWLR-H1} (Any $k_1$,$k_2$) & \textbf{7} & \textbf{8} & \textbf{18} & 539 & \textbf{6} & \textbf{5} & \textbf{6} & \textbf{5} \\
        \textbf{PWLR-OPT} (Any $k_1$,$k_2$) & \textbf{7} & \textbf{10} & \textbf{10} & \textbf{74} & \textbf{9} & \textbf{8} & \textbf{9} & \textbf{9} \\
        \hline
    \end{tabular}
    \end{adjustbox}
    \end{sc}
    \end{scriptsize}
    \end{center}
    \vskip -0.1in
\end{table*}

\section{Experiments}
\label{section: experiments}
We implement the PWLR framework in Python and execute experiments on classifying data sets of finite graphs. Tables \ref{tab:results} and \ref{tab:dimension} list the classification results and dimensions of representations obtained from the PWLR procedure and contemporary graph embedding techniques. 

\textbf{[Data sets]} \; \;
We classify cheminformatics graph data sets with discrete and continuous features \cite{KKMMN2016}. For classifying graphs with discrete node labels, we choose MUTAG, PTC, NCI, PROTEINS, and DD data sets. All discrete node labels are one-hot encoded as real coordinate vectors. For classifying graphs with continuous attributes, we choose PROTEINS, BZR, COX2, BZR-MD, and COX2-MD data sets. A full table of classification results can be found in Appendix \ref{appendix: experimental_procedures}.

\textbf{[Procedures]} \; \; 
For implementing the PWLR embedding scheme, we choose two numbers of iterations $k_1$ and $k_2$ from $0$ to $29$. We denote by ``PWLR-H0'' and ``PWLR-H1'' the vectors obtained from (\ref{eq:first_representation}), ``PWLR-OPT-H0'' and ``PWLR-OPT-H1'' the vectors obtained from (\ref{eq:second_representation}), and by ``PWLR-H0+H1'' and ``PWLR-OPT-H0+H1'' the vectors obtained from concatenating the two vectors ``PWLR-H0'' and ``PWLR-H1'' (``PWLR-OPT-H0'' and ``PWLR-OPT-H1'', respectively). We implemented 10 iterations of 10-fold cross validations for classifying graph data sets, along with inner 5-fold cross validations over the training sets for tuning the hyperparmeters using grid search. As a classifier, we use the random forest classifier \cite{Br01} to effectively assess the contributions of the architecture of the embedding framework. The optimal number of iterations $k_1$ and $k_2$ for classifying graphs are provided in Appendix \ref{appendix: experimental_procedures},

\textbf{[Results and Highlights]} \; \;
To evaluate the performance of the proposed algorithm in classifying graphs, we compare the PWLR scheme with WL kernel (WL) \cite{SS11}, Persistent WL representations (PWL) \cite{RBB19}, Wasserstein WL kernel (WWL) \cite{TG19}, graph kernels based on return probabilities of random walks (RetGK) \cite{ZW18}, hash graph kernels (HGK) \cite{MK16}, Fused Gromov-Wasserstein kernels (FGW) \cite{TC19}, filtration curves for graph representations (FC) \cite{BRB21}, the supervised version of graph convolutional kernel networks using subtree features (GCKN-subtree) \cite{CJM20}, Perslay \cite{CC20}, and DeepMap (DM) \cite{YA20}.. The first two procedures can represent graphs with discrete node labels. All other techniques can process graphs with both discrete and continuous node features, and weights on edges. All contemporary classification results are imported from available results recorded in pre-existing publications. Comparisons in classification results obtained from other graph kernels or graph neural networks (GNN) can be found in Appendix \ref{appendix: experimental_procedures}. 
Table \ref{tab:results} records the highest averages and standard deviations obtained from each graph classification techniques. Experimental results suggest that our PWLR embedding framework possesses two key empirical merits for representing graphs. 

$\circ$ \textbf{Low-dimensional Embeddings:}  The proposed scheme constructs a collection of low-dimensional representations independent from the choice of training data sets and the number of iterations $k_1$ and $k_2$.  As shown in Table \ref{tab:dimension}, it constructs fixed low-dimensional graph representations which produce comparable results to contemporary techniques. We especially notice that the ``PWLR-OPT'' representations (\ref{eq:second_representation}) obtained for graphs in large data sets are of substantially smaller dimensions than those obtained from graph kernels or WL procedures. Representations obtained from graph kernel based methods are characterized by inner products between embeddings of graphs and those from the training data set. As such, they heavily depend on the choice of training data. In addition, subsequent iterations of WL procedure substantially increases the number of obtainable distinct node labels, thus inevitably increasing the dimensions of representations.

$\circ$ \textbf{Classification Results:} The PWLR embedding framework produces comparable results to state-of-the-art techniques in classifying graphs with discrete node labels, and enhances these techniques in classifying graphs with continuous node attributes. While the PWLR embedding framework falls short in classifying some data sets with discrete node labels, all the proposed representations are of low dimensions, a property difficult to guarantee from other graph kernel techniques. As for other graph data sets with both discrete and continuous attributes along with weights on edges, the PWLR embedding framework improves or produces comparable state-of-the-art classification results.

\section{Conclusion}
\label{section: conclusion}
The problem of embedding graphs to real vector spaces requires a careful approach to how properties of graphs can be adequately incorporated to their representations. We address this question by introducing a novel mathematical framework for graph representations which guarantees stability with respect to graph perturbations as well as incorporates local topological features, node degrees, and global topological invariants. Experimental results suggest that our PWLR embedding framework provides low-dimensional representations effective for classifying graphs with both discrete and continuous node features. Meanwhile, the PWLR scheme implicitly assumes that the graphs are undirected and stochastically fixed. Future research may hence focus on extending the proposed framework to such graphs.

\section*{Acknowledgments}
We sincerely thank the reviewers for providing constructive and enlightening comments during the referee process. Sun Woo Park, Yun Young Choi, and Youngho Woo were supported by the National Institute for Mathematical
Sciences (NIMS) grant funded by the Korean Government (MSIT) B22920000. Dosang Joe was supported by the National Institute for Mathematical Sciences (NIMS) grant funded by the Korean Government (MSIT) B22810000.


\bibliography{pmlr_template.bbl}

\begin{thebibliography}{48}
\providecommand{\natexlab}[1]{#1}
\providecommand{\url}[1]{\texttt{#1}}
\expandafter\ifx\csname urlstyle\endcsname\relax
  \providecommand{\doi}[1]{doi: #1}\else
  \providecommand{\doi}{doi: \begingroup \urlstyle{rm}\Url}\fi

\bibitem[Baek et~al.(2021)Baek, DiMaio, Anischenko, Dauparas, Ovchinnilkov,
  Lee, Wang, Cong, Kinch, Schaeffer, and et~al]{BD21}
Baek, M., DiMaio, F., Anischenko, I., Dauparas, J., Ovchinnilkov, S., Lee,
  G.~R., Wang, J., Cong, Q., Kinch, L.~N., Schaeffer, R.~D., and et~al.
\newblock Accurate prediction of protein structures and interactions using a
  three-track network.
\newblock \emph{Science}, 373\penalty0 (6557):\penalty0 871--876, 2021.

\bibitem[Bodnar et~al.(2021)Bodnar, Frasca, Otter, Wang, Li\`{o}, Montufar, and
  Bronstein]{BF21}
Bodnar, C., Frasca, F., Otter, N., Wang, Y., Li\`{o}, P., Montufar, G.~F., and
  Bronstein, M.
\newblock Weisfeiler and lehman go cellular: Cw networks.
\newblock In \emph{Advances in Neural Information Processing Systems},
  volume~34, pp.\  2625--2640. Curran Associates, Inc., 2021.
\newblock URL
  \url{https://proceedings.neurips.cc/paper/2021/file/157792e4abb490f99dbd738483e0d2d4-Paper.pdf}.

\bibitem[Borgwardt \& Kriegel(2005)Borgwardt and Kriegel]{BK05}
Borgwardt, K.~M. and Kriegel, H.-P.
\newblock Shortest-path kernels on graphs.
\newblock In \emph{Proceedings of the Fifth IEEE International Conference on
  Data Mining}, pp.\  74–81, USA, 2005. IEEE Computer Society.
\newblock \doi{10.1109/ICDM.2005.132}.
\newblock URL \url{https://doi.org/10.1109/ICDM.2005.132}.

\bibitem[Borgwardt et~al.(2005)Borgwardt, Ong, Sch\''{o}nauer, Vishwanathan,
  Smola, and Kriegel]{BS05}
Borgwardt, K.~M., Ong, C.~S., Sch\''{o}nauer, S., Vishwanathan, S., Smola,
  A.~J., and Kriegel, H.-P.
\newblock Protein function prediction via graph kernels.
\newblock \emph{Bioinformatics}, 21:\penalty0 i47--i56, 2005.

\bibitem[Breiman(2001)]{Br01}
Breiman, L.
\newblock Random forests.
\newblock \emph{Machine learning}, 45:\penalty0 5--32, 2001.
\newblock URL \url{https://doi.org/10.1023/A:1010933404324}.

\bibitem[Carlsson(2009)]{Ca09}
Carlsson, G.
\newblock Topology and data.
\newblock \emph{Bulletin of the American Mathematical Society}, 46\penalty0
  (2):\penalty0 255--308, 2009.

\bibitem[Carriere et~al.(2020)Carriere, Chazal, Ike, Lacombe, Royer, and
  Umeda]{CC20}
Carriere, M., Chazal, F., Ike, Y., Lacombe, T., Royer, M., and Umeda, Y.
\newblock Perslay: A neural network layer for persistence diagrams and new
  graph topological signatures.
\newblock In \emph{Proceedings of the Twenty Third International Conference on
  Artificial Intelligence and Statistics}, volume 108, pp.\  2786--2796, 2020.

\bibitem[Chazal et~al.(2016)Chazal, de~Silva, Glisse, and Oudot]{CS16}
Chazal, F., de~Silva, V., Glisse, M., and Oudot, S.
\newblock \emph{The structure and stability of persistence modules}.
\newblock Springer International Publishing, 2016.

\bibitem[Chen et~al.(2020)Chen, Jacob, and Mairal]{CJM20}
Chen, D., Jacob, L., and Mairal, J.
\newblock Convolutional kernel networks for graph-structured data.
\newblock In \emph{Proceedings of the 37th International Conference on Machine
  Learning}, volume 119, pp.\  1576--1586. PMLR, 2020.

\bibitem[Debnath et~al.(1991)Debnath, Lopez~de Compadre, Debnath, Shusterman,
  and HAnsch]{DL91}
Debnath, A.~K., Lopez~de Compadre, R.~L., Debnath, G., Shusterman, A.~J., and
  HAnsch, C.
\newblock Structure-activity relationship of mutagenic aromatic and
  heteroaromatic nitro compounds. correlation with molecular orbital energies
  and hydrophobicity.
\newblock \emph{Journal of medicinal chemistry}, 34\penalty0 (2):\penalty0
  786--797, 1991.

\bibitem[Dobson \& Doig(2003)Dobson and Doig]{DD03}
Dobson, P. and Doig, A.
\newblock Distinguishing enzyme structures from non-enzymes without alignments.
\newblock \emph{Journal of Molecular Biology}, 330\penalty0 (4):\penalty0
  771--783, 2003.

\bibitem[Du et~al.(2019)Du, Hou, Salakhutdinov, Poczos, Wang, and Xu]{SK19}
Du, S.~S., Hou, K., Salakhutdinov, R.~R., Poczos, B., Wang, R., and Xu, K.
\newblock Graph neural tangent kernel: Fusing graph neural networks with graph
  kernels.
\newblock In \emph{Advances in Neural Information Processing Systems},
  volume~32. Curran Associates, Inc., 2019.
\newblock URL
  \url{https://proceedings.neurips.cc/paper/2019/file/663fd3c5144fd10bd5ca6611a9a5b92d-Paper.pdf}.

\bibitem[Edelsbrunner \& Harer(2010)Edelsbrunner and Harer]{EH10}
Edelsbrunner, H. and Harer, J.
\newblock \emph{Computational topology: An introduction}.
\newblock American Mathematical Society, 2010.

\bibitem[Edelsbrunner et~al.(2002)Edelsbrunner, Letscher, and
  Zomorodian]{ELZ02}
Edelsbrunner, H., Letscher, D., and Zomorodian, A.
\newblock topological persistence and simplification.
\newblock \emph{Discrete and computational geometry}, 2002.

\bibitem[Feragen et~al.(2013)Feragen, Kasenburg, Petersen, Bruijne, and
  Borgwardt]{FK13}
Feragen, A., Kasenburg, N., Petersen, J., Bruijne, M.~d., and Borgwardt, K.
\newblock Scalable kernels for graphs with continuous attributes.
\newblock \emph{Advances in Neural Information Processing Systems}, 32, 2013.
\newblock URL
  \url{https://papers.nips.cc/paper/2013/hash/a2557a7b2e94197ff767970b67041697-Abstract.html}.

\bibitem[Geerts et~al.(2021)Geerts, Mazowiecki, and Per\'{e}z]{GMP21}
Geerts, F., Mazowiecki, F., and Per\'{e}z, G.~A.
\newblock Let’s agree to degree: Comparing graph convolutional networks in
  the message-passing framework.
\newblock In \emph{Proceedings of the 38th International Conference on Machine
  Learning}, volume 139, pp.\  3640--3649. PMLR, 2021.

\bibitem[Hatcher(2002)]{Ha02}
Hatcher, A.
\newblock \emph{Algebraic Topology}.
\newblock Cambridge University Press, 2002.

\bibitem[Helma et~al.(2001)Helma, King, Kramer, and Srinivasan]{HK01}
Helma, C., King, R.~D., Kramer, S., and Srinivasan, A.
\newblock The predictive toxicology challenge 2000-2001.
\newblock \emph{Bioinformatics}, 17\penalty0 (1):\penalty0 107--108, 2001.

\bibitem[Hofer et~al.(2017)Hofer, Kwitt, Niethammer, and Uhl]{CR17}
Hofer, C., Kwitt, R., Niethammer, M., and Uhl, A.
\newblock Deep learning with topological signatures.
\newblock In Guyon, I., Luxburg, U.~V., Bengio, S., Wallach, H., Fergus, R.,
  Vishwanathan, S., and Garnett, R. (eds.), \emph{Advances in Neural
  Information Processing Systems}, volume~30. Curran Associates, Inc., 2017.
\newblock URL
  \url{https://proceedings.neurips.cc/paper/2017/file/883e881bb4d22a7add958f2d6b052c9f-Paper.pdf}.

\bibitem[Hu et~al.(2014)Hu, Rustamov, and Guibas]{HRG14}
Hu, N., Rustamov, R.~M., and Guibas, L.~J.
\newblock Stable and informative spectral signatures for graph matching.
\newblock \emph{2014 IEEE Conference on Computer Vision and Pattern
  Recognition}, pp.\  2313--2320, 2014.

\bibitem[Isaacson \& Madsen(1976)Isaacson and Madsen]{IM76}
Isaacson, D.~L. and Madsen, R.~W.
\newblock \emph{Markov chains: theory and applications}.
\newblock John wiley and Sons, 1976.

\bibitem[Kashima et~al.(2003)Kashima, Tsuda, and Inokuchi]{HK03}
Kashima, H., Tsuda, K., and Inokuchi, A.
\newblock Marginalized kernels between labeled graphs.
\newblock In \emph{Proceedings of the Twentieth International Conference on
  International Conference on Machine Learning}, ICML'03, pp.\  321–328. AAAI
  Press, 2003.

\bibitem[Kersting et~al.(2016)Kersting, Kriege, Morris, Mutzel, and
  Neumann]{KKMMN2016}
Kersting, K., Kriege, N.~M., Morris, C., Mutzel, P., and Neumann, M.
\newblock Benchmark data sets for graph kernels, 2016.
\newblock URL \url{http://graphkernels.cs.tu-dortmund.de}.

\bibitem[Kipf \& Welling(2017)Kipf and Welling]{KW17}
Kipf, T.~N. and Welling, M.
\newblock Semi-supervised classification with graph convolutional networks.
\newblock In \emph{5th International Conference on Learning Representations,
  {ICLR} 2017, Toulon, France, April 24-26, 2017, Conference Track
  Proceedings}. OpenReview.net, 2017.
\newblock URL \url{https://openreview.net/forum?id=SJU4ayYgl}.

\bibitem[Lawler(2006)]{La06}
Lawler, G.
\newblock \emph{Introduction to Stochastic Processes}.
\newblock Chapman and Hall, 2006.

\bibitem[Lovasz(1993)]{Lo93}
Lovasz, L.
\newblock Random walks on graphs: a survey.
\newblock \emph{Combinatorics}, pp.\  1--46, 1993.

\bibitem[Morris et~al.(2016)Morris, Kriege, Kersting, and Mutzel]{MK16}
Morris, C., Kriege, N.~M., Kersting, K., and Mutzel, P.
\newblock Faster kernel for graphs with continuous attributes via hashing.
\newblock In \emph{IEEE International Conference on Data Mining (ICDM), 2016},
  pp.\  1095--1100, 2016.

\bibitem[Morris et~al.(2017)Morris, Kersting, and Mutzel]{CK17}
Morris, C., Kersting, K., and Mutzel, P.
\newblock Glocalized weisfeiler-lehman graph kernels: Global-local feature maps
  of graphs.
\newblock In \emph{2017 IEEE International Conference on Data Mining (ICDM)},
  pp.\  327--336, 2017.
\newblock \doi{10.1109/ICDM.2017.42}.

\bibitem[Niepert et~al.(2016)Niepert, Ahmed, and Kutzkov]{NAK16}
Niepert, M., Ahmed, M., and Kutzkov, K.
\newblock Learning convolutional neural networks for graphs.
\newblock In \emph{Proceedings of The 33rd International Conference on Machine
  Learning}, volume~48, pp.\  2014--2023. PMLR, 2016.

\bibitem[Nikolentzos \& Vazirgiannis(2020)Nikolentzos and Vazirgiannis]{NV20}
Nikolentzos, G. and Vazirgiannis, M.
\newblock Random walk graph neural networks.
\newblock In Larochelle, H., Ranzato, M., Hadsell, R., Balcan, M., and Lin, H.
  (eds.), \emph{Advances in Neural Information Processing Systems}, volume~33,
  pp.\  16211--16222. Curran Associates, Inc., 2020.
\newblock URL
  \url{https://proceedings.neurips.cc/paper/2020/file/ba95d78a7c942571185308775a97a3a0-Paper.pdf}.

\bibitem[O'Bray et~al.(2021)O'Bray, Rieck, and Borgwardt]{BRB21}
O'Bray, L., Rieck, B., and Borgwardt, K.
\newblock Filtration curves for graph representation.
\newblock In \emph{Proceedings of the 27th ACM SIGKDD Conference on Knowledge
  Discovery and Data Mining}, KDD '21, pp.\  1267–1275. Association for
  Computing Machinery, 2021.

\bibitem[Rieck et~al.(2019)Rieck, Bock, and Borgwardt]{RBB19}
Rieck, B., Bock, C., and Borgwardt, K.
\newblock A persistent {W}eisfeiler--{L}ehman procedure for graph
  classification.
\newblock In Chaudhuri, K. and Salakhutdinov, R. (eds.), \emph{Proceedings of
  the 36th International Conference on Machine Learning}, volume~97, pp.\
  5448--5458, Long Beach, California, USA, 2019. PMLR.
\newblock URL \url{http://proceedings.mlr.press/v97/rieck19a.html}.

\bibitem[Schweitzer(1968)]{Sc68}
Schweitzer, P.
\newblock Perturbation theory and finite markov chains.
\newblock \emph{Journal of Applied Probability}, 5\penalty0 (3):\penalty0
  401--413, 1968.

\bibitem[Senior et~al.(2020)Senior, Evans, Jumper, Kirkpatrick, Sifre, Green,
  Qin, Zidek, Nelson, Bridgland, and et~al]{SE20}
Senior, A.~W., Evans, R., Jumper, J., Kirkpatrick, J., Sifre, L., Green, T.,
  Qin, C., Zidek, A., Nelson, A.~W., Bridgland, and et~al.
\newblock Improved protein structure prediction using potentials from deep
  learning.
\newblock \emph{Nature}, 577\penalty0 (7792):\penalty0 706--710, 2020.

\bibitem[Shervashidze et~al.(2011)Shervashidze, Schweitzer, van Leeuwen,
  Mehlhorn, and Borgwardt]{SS11}
Shervashidze, N., Schweitzer, P., van Leeuwen, E.~J., Mehlhorn, K., and
  Borgwardt, K.~M.
\newblock Weisfeiler-lehman graph kernels.
\newblock \emph{Journal of Machine Learning Research}, 12:\penalty0 2539--2561,
  2011.

\bibitem[Stokes et~al.(2020)Stokes, Yang, Swanson, Jin, Cubillos-Ruiz, Donghia,
  MacNair, French, Carfrae, Bloom-Ackerman, and et~al]{SY20}
Stokes, J.~M., Yang, K., Swanson, K., Jin, W., Cubillos-Ruiz, A., Donghia,
  N.~M., MacNair, C.~R., French, S., Carfrae, L.~A., Bloom-Ackerman, Z., and
  et~al.
\newblock A deep learning approach to antibiotic discovery.
\newblock \emph{Cell}, 180\penalty0 (4):\penalty0 401--413, 2020.

\bibitem[Sugiyama \& Borgwardt(2015)Sugiyama and Borgwardt]{SB16}
Sugiyama, M. and Borgwardt, K.
\newblock Halting in random walk kernels.
\newblock In Cortes, C., Lawrence, N., Lee, D., Sugiyama, M., and Garnett, R.
  (eds.), \emph{Advances in Neural Information Processing Systems}, volume~28.
  Curran Associates, Inc., 2015.
\newblock URL
  \url{https://proceedings.neurips.cc/paper/2015/file/31b3b31a1c2f8a370206f111127c0dbd-Paper.pdf}.

\bibitem[Sutherland et~al.(2003)Sutherland, O'brien, and Weaver]{SO03}
Sutherland, J.~J., O'brien, L.~A., and Weaver, D.~F.
\newblock Spline-fitting with a genetic algorithm: A method for developing
  classification structure-activity relationships.
\newblock \emph{Journal of chemical information and computer sciences},
  43\penalty0 (6):\penalty0 1906--1915, 2003.

\bibitem[Titouan et~al.(2019)Titouan, Courty, Tavenard, Laetitia, and
  Flamary]{TC19}
Titouan, V., Courty, N., Tavenard, R., Laetitia, C., and Flamary, R.
\newblock Optimal transport for structured data with application on graphs.
\newblock In Chaudhuri, K. and Salakhutdinov, R. (eds.), \emph{Proceedings of
  the 36th International Conference on Machine Learning}, volume~97, pp.\
  6275--6284, Long Beach, California, USA, 09--15 Jun 2019. PMLR.
\newblock URL \url{http://proceedings.mlr.press/v97/titouan19a.html}.

\bibitem[Togninalli et~al.(2019)Togninalli, Ghisu, Llinares-L{\'o}pez, Rieck,
  and Borgwardt]{TG19}
Togninalli, M., Ghisu, E., Llinares-L{\'o}pez, F., Rieck, B., and Borgwardt, K.
\newblock Wasserstein weisfeiler--lehman graph kernels.
\newblock In Wallach, H., Larochelle, H., Beygelzimer, A., d'Alch\'{e}{-}Buc,
  F., Fox, E., and Garnett, R. (eds.), \emph{Advances in Neural Information
  Processing Systems~32~(NeurIPS)}, pp.\  6436--6446. Curran Associates, Inc.,
  2019.

\bibitem[Velickovic et~al.(2018)Velickovic, Cucurull, Casanova, Romero,
  Li{\`{o}}, and Bengio]{VC18}
Velickovic, P., Cucurull, G., Casanova, A., Romero, A., Li{\`{o}}, P., and
  Bengio, Y.
\newblock Graph attention networks.
\newblock In \emph{6th International Conference on Learning Representations,
  {ICLR} 2018, Vancouver, BC, Canada, April 30 - May 3, 2018, Conference Track
  Proceedings}. OpenReview.net, 2018.
\newblock URL \url{https://openreview.net/forum?id=rJXMpikCZ}.

\bibitem[Vishwanathan et~al.(2010)Vishwanathan, Schraudolph, Kondor, and
  Borgwardt]{VS10}
Vishwanathan, S., Schraudolph, N.~N., Kondor, R., and Borgwardt, K.~M.
\newblock Graph kernels.
\newblock \emph{Journal of Machine Learning Research}, 11:\penalty0 1201--1242,
  2010.

\bibitem[Wale \& Karypic(2008)Wale and Karypic]{WK06}
Wale, N. and Karypic, G.
\newblock Comparison of descriptor spaces for chemical compound retrieval and
  classification.
\newblock \emph{Proceedings of the International Conference on Data Mining},
  14:\penalty0 347--375, 2008.
\newblock URL \url{https://doi.org/10.1007/s10115-007-0103-5}.

\bibitem[Weisfeiler \& Lehman(1968)Weisfeiler and Lehman]{WL68}
Weisfeiler, B. and Lehman, A.
\newblock {A reduction of a graph to a canonical form and an algebra arising
  during this reduction}.
\newblock \emph{Nauchno-Technicheskaya Informatsia, Ser.2}, 9, 1968.

\bibitem[Xu et~al.(2019)Xu, Hu, Leskovec, and Jegelka]{XH19}
Xu, K., Hu, W., Leskovec, J., and Jegelka, S.
\newblock How powerful are graph neural networks?
\newblock In \emph{7th International Conference on Learning Representations,
  {ICLR} 2019, New Orleans, LA, USA, May 6-9, 2019}. OpenReview.net, 2019.
\newblock URL \url{https://openreview.net/forum?id=ryGs6iA5Km}.

\bibitem[Ye et~al.(2020)Ye, Askarisichani, Jones, and Singh]{YA20}
Ye, W., Askarisichani, O., Jones, A., and Singh, A.
\newblock Learning deep graph representations via convolutional neural
  networks.
\newblock \emph{IEEE Transactions on Knowledge and Data Engineering}, pp.\
  1--1, 2020.
\newblock \doi{10.1109/TKDE.2020.3014089}.

\bibitem[Zhang et~al.(2018{\natexlab{a}})Zhang, Cui, Neumann, and Chen]{ZC18}
Zhang, M., Cui, Z., Neumann, M., and Chen, Y.
\newblock An end-to-end deep learning architecture for graph classification.
\newblock \emph{Proceedings of the AAAI Conference on Artificial Intelligence},
  32\penalty0 (1), 2018{\natexlab{a}}.
\newblock URL \url{https://ojs.aaai.org/index.php/AAAI/article/view/11782}.

\bibitem[Zhang et~al.(2018{\natexlab{b}})Zhang, Wang, Xiang, Huang, and
  Nehorai]{ZW18}
Zhang, Z., Wang, M., Xiang, Y., Huang, Y., and Nehorai, A.
\newblock Retgk: Graph kernels based on return probabilities of random walks.
\newblock In Bengio, S., Wallach, H., Larochelle, H., Grauman, K.,
  Cesa-Bianchi, N., and Garnett, R. (eds.), \emph{Advances in Neural
  Information Processing Systems}, volume~31. Curran Associates, Inc.,
  2018{\natexlab{b}}.
\newblock URL
  \url{https://proceedings.neurips.cc/paper/2018/file/7f16109f1619fd7a733daf5a84c708c1-Paper.pdf}.

\end{thebibliography}
\bibliographystyle{iclr2022_workshop}

\newpage
\appendix
\onecolumn


\begin{center}
\begin{Huge}
\textbf{Appendix}
\end{Huge}
\end{center}

Provided below is a table of contents for the appendix of this manuscript.
\begin{itemize}
    \item Appendix \ref{appendix: notations}: We provide a table of notations used throughout the manuscript.
    \item Appendix \ref{appendix: experimental_procedures}: Full classification results obtained from utilizing the PWLR embedding framework as well as additional discussions are provided.
    \item Appendix \ref{appendix: examples}: We give relevant examples to elaborate how one can compute representations of graphs using the PWLR embedding scheme.
    \item Appendix \ref{appendix: mathematical_background}: We give an exposition on mathematical results on random walks over graphs and persistent homological techniques, which are key ingredients for proving the main results of this paper.
    \item Appendix \ref{appendix: technical_proofs}: We prove Proposition \ref{proposition: normalized_wl}, Theorem \ref{theorem: three_properties}, and Theorem \ref{theorem:stability} in this section. Time complexity required for computing the PWLR representations are also discussed.
\end{itemize}

\section{Notations}
\label{appendix: notations}
 
 Table \ref{tab:notation} provides a list of notations used throughout the paper.
\begin{table*}[ht]
    \caption{A table of notations}
    \label{tab:notation}
    \vskip 0.15in
    \begin{center}
    \begin{small}
    \begin{sc}
    \begin{adjustbox}{width=\textwidth}
    \begin{tabular}{c|c}
        \hline
        Notation & Meaning  \\
        \hline
        \hline
        $G$ & An undirected finite graph without self loops \\
        \hline
        $V = V(G)$ & The set of nodes of $G$ \\
        \hline
        $E = E(G)$ & The set of edges of $G$ \\
        \hline
        $L: V \to \mathbb{R}^l_{>0}$ & The $l$-dimensional node attributes or labels of $G$ \\
        \hline
        $A := \{a_{i,j}\}_{i,j=1}^{|V|}$ & The weighted adjacency matrix of $G$ \\
        \hline
        $a_{i,j}$ & The weight over the edge from node $i$ to node $j$ \\
        \hline
        $\overline{d}_i$ & The unweighted degree of node $i$ \\
        \hline
        $D := \text{Diag}(d_i)_{i=1}^{|V|}$ & A matrix of weighted node degrees of $G$ \\
        \hline
        $d_i := \sum_{j=1}^{|V|} a_{i,j}$ & The weighted degree of node $i$, as a sum of $a_{i,j}$'s. \\
        \hline
        $X$ & The matrix of concatenated node labels of $G$ \\
        \hline
        $M_G$ & The normalized adjacency matrix of $G$ (\ref{equation:transition_matrix}) \\
        \hline
        $\nu$ & The vector whose entries are all equal to $1$ \\
        \hline
        $\pi_G$ & The stationary distribution of the transition matrix of $G$ \\
        \hline
        $\mu_{2,G}$ & The second largest eigenvalue of the transition matrix $M_G$ \\
        \hline
        $\epsilon$ & The perturbation of two transition matrices $M_G$ and $M_{G'}$ \\
        \hline
        $h_V$ & The height function defined over the set of nodes $V = V(G)$ (\ref{equation:height_function}) \\
        \hline
        $h_E$ & The height function defined over the set of edges $E = E(G)$ (\ref{equation:height_function}) \\
        \hline
        $\{G^{[i]}\}$ & A sequence of nested subgraphs constructed from the height functions $h_V, h_E$ \\
        \hline
        $\overline{d}_e$ & A tuple of unweighted degrees of two nodes connected by the edge $e$ \\
        \hline
        $\overline{D}_E$ & The set of tuples of unweighted degrees of two nodes connected by edges in $E$ (\ref{equation:set_unweighted_degrees}) \\
        \hline
        $\beta_0$ & The rank of the 0-th homology group of $G$, i.e. the number of its connected components \\
        \hline
        $\beta_1$ & The rank of the 1-st homology group of $G$, i.e. the number of its cycles \\
        \hline
        $\varphi_{H_0}^{[k_1,k_2]}$ & The 0-th homology PWLR representation of $G$ from (\ref{eq:first_representation})\\
        & with predetermined number of iterations $k_1$ and $k_2$ \\
        \hline
        $\varphi_{H_1}^{[k_1,k_2]}$ & The 1-st homology PWLR representation of $G$ from (\ref{eq:first_representation}) \\
        & with predetermined number of iterations $k_1$ and $k_2$ \\
        \hline
        $\varphi_{H_0, Opt}^{[k_1,k_2]}$ & The 0-th homology reduced PWLR representation of $G$ from (\ref{eq:second_representation})\\
        & with predetermined number of iterations $k_1$ and $k_2$ \\
        \hline
        $\varphi_{H_1, Opt}^{[k_1,k_2]}$ & The 1-st homology reduced PWLR representation of $G$ from (\ref{eq:second_representation})\\
        & with predetermined number of iterations $k_1$ and $k_2$ \\
        \hline
    \end{tabular}
    \end{adjustbox}
    \end{sc}
    \end{small}
    \end{center}
    \vskip -0.1in
\end{table*}

\section{Experimental Procedures}
\label{appendix: experimental_procedures}


\subsection{Experimental Settings}
\label{subsection:full_results}

The graph data sets used for assessing the performance of the PWLR embedding framework can be classified into three groups. The first group, consisting of MUTAG \cite{DL91}, PTC-MM, PTC-MR, PTC-FM, and PTC-FR \cite{HK01}, represents small graph data sets with discrete node labels. The second group represents large graph data sets with discrete node labels, which includes NCI-1, NCI-109 \cite{WK06}, PROTEINS \cite{BS05}, and DD \cite{DD03} data sets. The last group is the data set of graphs with continuous node attributes or edge weights, which consists of PROTEINS \cite{BS05}, BZR, COX2, BZR-MD, and COX2-MD data sets \cite{SO03}. These data sets can be found in the TU Dortmund online benchmark data set repository \cite{KKMMN2016}.

We compared the performance of the PWLR embedding framework to three classes of contemporary techniques which have been widely used for classifying graph data sets. The first class represents WL based procedures for representing graphs with discrete node labels, including the classical WL kernel (WL) \cite{SS11} and Persistent WL procedure (PWL) \cite{RBB19}. The second class of techniques are graph kernels which can embed graphs with continuous node attributes, such as  Wasserstein WL kernel (WWL) \cite{TG19}, graph kernels based on return probabilities of random walks (RetGK) \cite{ZW18}, hash graph kernels (HGK) \cite{MK16}, GraphHopper kernels (GH) \cite{FK13}, Fused Gromov-Wasserstein kernels (FGW) \cite{TC19}, and Filtration Curves for graph representations (FC) \cite{BRB21}. The third class of techniques are graph neural networks (or message passing neural networks), from which we choose PatchySan (PSCN) \cite{NAK16}, deep graph convolutional neural networks (DGCNN) \cite{ZC18}, the supervised version of graph convolutional kernel networks using subtree features (GCKN-subtree-sup) \cite{CJM20}, Perslay \cite{CC20}, Graph Isomorphism Network (GIN) \cite{XH19}, Graph Neural Tangent Kernel (GNTK) \cite{SK19}, Random Walk Graph Neural Networks (RWNN) \cite{NV20}, and DeepMap (DM) \cite{YA20}. All classification results shown in this paper are obtained from respective publications.

To evaluate the significance of representations obtained from PWLR formulation in classifying graphs, we use the random forest classifier \cite{Br01} to compare the proposed algorithm to other state-of-the-art methods. No normalization or pre-processing algorithms were used to alter these representations obtained prior to applying the classifier. We performed 10 iterations of 10-fold cross validation, along with an inner 5-fold cross validation on each training set for tuning hyperparameters using grid search. Averages and standard deviations obtained from these 10 iterations are recorded. As for the hyperparameters, we choose the number of iterations $k_1$ and $k_2$ to be between $0$ and $29$, norm $p = 1$, bias $\tau = 1$, and $T \in [10,25,50,100,150,200]$ for the number of trees in the random forest classifier.

\subsection{Small Graph data sets}

\begin{table}[!htbp]
    \caption{Classification results obtained from small graph data sets with discrete node labels. Cells notated as N/A indicate graph classification techniques which do not report classification results on the given graph data set. For PWLR embedding methods, we indicate which iteration numbers $k_1$ and $k_2$ are used to obtain the recorded results. Some classification results obtained from graph neural networks do not specify which specific PTC data sets are used, hence are displayed in concatenated cells. All classification results other than the proposed method are used from available classification results obtained from pre-existing publications.}
    \label{appendix:discrete_small_results}
    \vskip 0.15in
    \begin{center}
    \begin{small}
    \begin{sc}
    \begin{adjustbox}{width=\textwidth}
    \begin{tabular}{cccccc}
            \hline
             & MUTAG & PTC-MM & PTC-MR & PTC-FM & PTC-FR \\
             \hline
             \hline
            WL & $88.72 \pm 1.11$ & $67.28 \pm 0.97$ & $63.12 \pm 1.44$ & $64.80 \pm 0.85$ & $67.64 \pm 0.74$ \\
            WWL & $87.27 \pm 1.50$ & N/A & $\mathbf{66.31 \pm 1.21}$ & N/A & N/A \\
            PWL-H0 & $86.10 \pm 1.37$ & $68.40 \pm 1.17$ & $63.07 \pm 1.68$ & $64.47 \pm 1.84$ & $67.30 \pm 1.50$ \\
            PWL-H1 & $90.51 \pm 1.34$ & $\mathbf{68.57 \pm 1.76}$ & $64.02 \pm 0.82$ & $\mathbf{65.78 \pm 1.22}$ & $67.15 \pm 1.09$ \\
            RetGK-I,II & $90.30 \pm 1.10$ & $67.90 \pm 1.40$ & $62.50 \pm 1.60$ & $63.90 \pm 1.30$ & $67.80 \pm 1.10$ \\
            FC & $87.31 \pm 0.66$ & \multicolumn{4}{c}{N/A} \\
            \hline
            \hline
            PSCN & $\mathbf{92.63} \pm \mathbf{4.21}$ & $56.58 \pm 9.01$ & $55.25 \pm 7.98$ & $58.38 \pm 9.27$ & $61.00 \pm 5.61$ \\
            DGCNN & $88.72 \pm 1.11$ & $62.12 \pm 14.1$ & $55.29 \pm 9.28$ & $60.29 \pm 6.69$ & $65.43 \pm 11.3$ \\
            GCKN-subtree-sup & $91.60 \pm 6.70$ & \multicolumn{4}{c}{$\mathbf{68.40 \pm 7.40}$} \\
            Perslay  & $89.80 \pm 0.90 $ & \multicolumn{4}{c}{N/A} \\ 
            GIN & $89.40 \pm 5.60$ & $67.19 \pm 7.41$ & $62.57 \pm 5.18$ & $64.22 \pm 2.36$ & $66.97 \pm 6.17$ \\
            GNTK & $90.00 \pm 8.50$ & $65.94 \pm 1.21$ & $58.32 \pm 1.00$ & $63.85 \pm 1.20$ & $66.97 \pm 0.56$ \\
            RWNN & $89.20 \pm 4.30$ & \multicolumn{4}{c}{N/A} \\
            DM & N/A & $\mathbf{69.59} \pm \mathbf{7.39}$ & $\mathbf{67.73} \pm \mathbf{6.61}$ & $65.16 \pm 5.62$ & $68.39 \pm 3.57$ \\
            \hline
            \hline
            \textbf{PWLR-H0} & $89.52 \pm 0.90$ & $65.75 \pm 1.22$ & $62.67 \pm 1.50$ & $62.93 \pm 1.45$ & $\mathbf{69.12 \pm 1.16}$ \\
            $(k_1,k_2)$ & (1,12) & (16,0) & (10,10) & (2,4) & (22,8) \\
            \hline
            \textbf{PWLR-H1} & $\mathbf{91.97 \pm 0.92}$ & $63.30 \pm 1.26$ & $58.92 \pm 1.16$ & $60.87 \pm 1.69$ & $67.04 \pm 0.67$ \\
            $(k_1,k_2)$ & (26,27) & (2,13) & (5,3) & (3,15) & (0,0) \\
            \hline
            \textbf{PWLR-H0+H1} & $89.47 \pm 1.51$ & $66.11 \pm 1.26$ & $63.89 \pm 1.75$ & $63.09 \pm 1.56$ & $\mathbf{68.75 \pm 1.39}$\\
            $(k_1,k_2)$ & (8,11) & (16,0) & (10,10) & (0,8) & (1,14) \\
            \hline
            \hline
            \textbf{PWLR-OPT-H0} & $89.73 \pm 1.01$ & $65.93 \pm 1.14$ & $60.08 \pm 2.11$ & $62.09 \pm 1.62$ & $64.02 \pm 0.99$ \\
            $(k_1,k_2)$ & (1,1) & (5,0) & (10,12) & (0,1) & (0,6) \\
            \hline
            \textbf{PWLR-OPT-H1} & $\mathbf{91.91 \pm 1.61}$ & $63.43 \pm 1.14$ & $58.21 \pm 1.51$ & $61.27 \pm 1.40$ & $66.19 \pm 1.23$ \\
            $(k_1,k_2)$ & \textbf{(16,27)} & (26,13) & (3,20) & (16,19) & (10,2) \\
            \hline
            \textbf{PWLR-OPT-H0+H1} & $89.17 \pm 0.84$ & $66.30 \pm 1.24$ & $62.78 \pm 0.97$ & $64.35 \pm 1.36$ & $65.73 \pm 1.16$ \\
            $(k_1,k_2)$ & (1,1) & (5,0) & (4,8) & (16,0) & (16,2) \\
            \hline
    \end{tabular}
    \end{adjustbox}
    \end{sc}
    \end{small}
    \end{center}

    \caption{Dimensions of representations processed in the Random Forest Classifier for classifying small graph data sets with discrete labels. Cells notated as - indicate graph data set which do not have the prescribed node features. The variable $h$ denotes the number of WL iteration procedures used for obtaining the representations. The asterix $*$ denotes that the obtained graph representations are subject to changes based on the choice of training data sets.}
    \label{appendix:small_discrete_dimension}
    \vskip 0.15in
    \begin{center}
    \begin{small}
    \begin{sc}
    \begin{tabular}{cccccc}
        \hline
        data sets & MUTAG & PTC-MM & PTC-MR & PTC-FM & PTC-FR \\
        \hline
        \hline
        Average \# nodes & 17.93 & 13.97 & 14.29 & 14.11 & 14.56\\
        Average \# edges & 19.79 & 14.32 & 14.69 & 14.48 & 15.00 \\
        \# Discrete Node Labels & 7 & 20 & 18 & 18 & 19 \\
        Dim. continuous node features & - & - & - & - & - \\
        \# Graphs & 188 & 336 & 344 & 349 & 351 \\
        \hline
        \hline
        Graph Kernels (10-fold CV)* & 169* & 302* & 311* & 314* & 315* \\
        \hline
        WL, PWL-H0 ($h=1$) & 40 & 152 & 148 & 149 & 148 \\
        WL, PWL-H0 ($h=10$) & 15,969 & 24,845 & 26,125 & 26,216 & 27,139 \\
        WL, PWL-H0 ($h=20$) & 41,825 & 58,081 & 61,240 & 61,212 & 63,404 \\
        \hline
        \hline
        \textbf{PWLR-H0} (Any $k_1$,$k_2$) & \textbf{28} & \textbf{64} & \textbf{64} & \textbf{64} & \textbf{64}\\
        \textbf{PWLR-H1} (Any $k_1$,$k_2$) & \textbf{7} & \textbf{8} & \textbf{8} & \textbf{8} & \textbf{8}\\
        \textbf{PWLR-OPT} (Any $k_1$,$k_2$) & \textbf{7} & \textbf{10} & \textbf{10} & \textbf{10} & \textbf{10} \\
        \hline
    \end{tabular}
    \end{sc}
    \end{small}
    \end{center}
    \vskip -0.1in
\end{table}

\begin{figure}[!htbp]
    \centering
    \subfloat[MUTAG data set]{\includegraphics[width=0.36\textwidth]{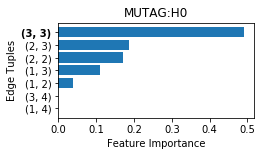} \includegraphics[width=0.36\textwidth]{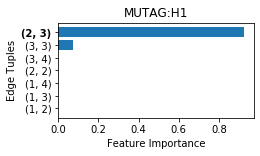}} \quad
    \subfloat[PTC-MM data set]{\includegraphics[width=0.36\textwidth]{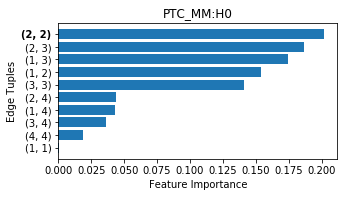} \includegraphics[width=0.36\textwidth]{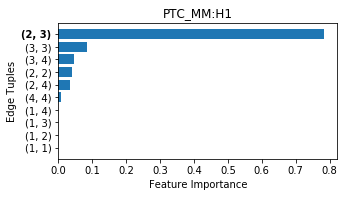}} \quad
    \subfloat[PTC-MR data set]{\includegraphics[width=0.36\textwidth]{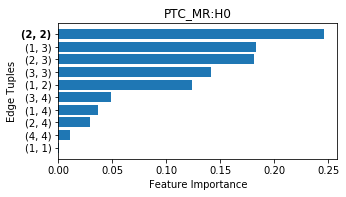} \includegraphics[width=0.36\textwidth]{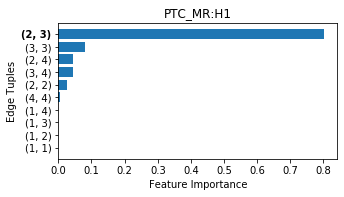}} \quad
    \subfloat[PTC-FM data set]{\includegraphics[width=0.36\textwidth]{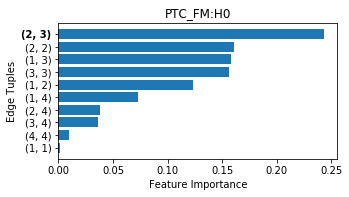} \includegraphics[width=0.36\textwidth]{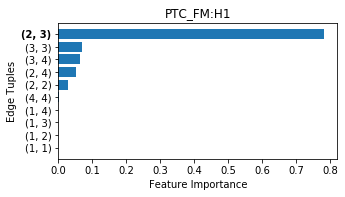}} \quad
    \subfloat[PTC-FR data set]{\includegraphics[width=0.36\textwidth]{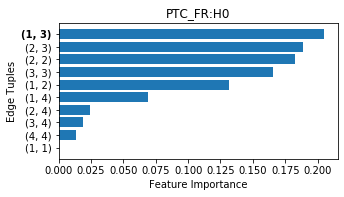} \includegraphics[width=0.36\textwidth]{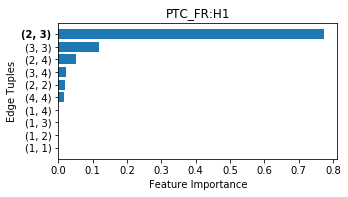}}
    \caption{Exemplary histograms of contributions each component of PWLR-OPT representations make to obtaining graph classification results for MUTAG and PTC data sets}
    \label{appendix:fig:small}
\end{figure}

Table \ref{appendix:discrete_small_results} records full classification results obtained from using graph representations constructed from the PWLR embedding framework and other contemporary techniques. We observe that the PWLR embedding framework enhances classification results in MUTAG and PTC-FR data sets, even surpassing those obtained from some examples of graph neural networks. For other PTC data sets, the PWLR embedding framework still manages to produce comparable results to contemporary techniques in classifying other variants of PTC data sets.

As explainable low-dimensional embeddings, the PWLR graph representations can pinpoint motifs of graph data sets that significantly contribute to classification results. Figure \ref{appendix:fig:small} demonstrates which tuples of unweighted node degrees associated to edges of $G$ contribute to classification of respective graph data sets. Here, the parameters $k_1$ and $k_2$ are number of iterations providing the highest classification scores from Table \ref{appendix:discrete_small_results}. The dimensions of the representations obtained from graph kernel techniques, WL, PWL, and PWLR embedding frameworks are outlined in Table \ref{appendix:small_discrete_dimension}. 

For MUTAG data sets, the classes of edges with tuples $(2,3)$ and $(3,3)$ which record the variations of number of cycles (PWLR-OPT-H1) completely determines the classification results. Interestingly, the clusters naturally observable from the embedded vectors are in correspondence to the provided binary graph labels (each notated in color blue and red). Attached in supplementary materials is a video clip which shows how clusters of graph representations appear as the number of iterations $k_1$ and $k_2$ vary.

As for the four variants of PTC data sets, the optimal number of iterations $k_1$ and $k_2$ as well as the contributions from the tuples of unweighted degrees alter with respect to graph data sets. This demonstrates that the PWLR framework provides a suitable collection of adaptable graph representations with suitable choices of parameters $k_1$ and $k_2$.

\subsection{Large Graph data sets}

Table \ref{appendix:discrete_large_results} lists the results obtained from classifying NCI1, NCI109, PROTEINS, and DD data sets. The PWLR embedding framework produces subpar classification results compared to state-of-the-art techniques, though they outperform a few graph neural networks in classifying NCI1 and NCI109 graph data sets. 

Nevertheless, we observe from Table \ref{appendix:large_discrete_dimension} that the dimensions of the proposed embeddings are substantially smaller than those obtained from graph kernels and WL procedures. For instance, a single iteration of WL procedure generates a vector representation of dimension equal to ``$253,231$'', a value too large to identify which components contribute the most to obtaining graph classification results. The reduced PWLR embedding formulation from (\ref{eq:second_representation}) cuts down the embedding dimension to $148$, at the cost of marginally reducing the classification results to $76.4\% - 76.7\%$. Thus, we may argue that the PWLR embedding scheme provides a more suitable form of representations than other graph kernels or WL procedures for identifying the relevant motifs for classifying DD data sets. 

As before, we illustrate in Figures \ref{appendix:fig:large:NCI} exemplary histograms of contributions each component of PWLR-Opt representations make in classifying NCI1 and NCI109 data sets. The parameters $k_1$ and $k_2$ are determined from the most optimal classification results shown in Table \ref{appendix:discrete_large_results}. 

\begin{table}[!htbp]
    \caption{Classification results obtained from large graph data sets with discrete node labels. Cells notated as N/A indicate graph classification techniques which do not report classification results on the given graph data set. For PWLR embedding methods, we indicate which iteration numbers $k_1$ and $k_2$ are used to obtain the recorded results. All classification results other than the proposed method are used from available classification results obtained from pre-existing publications.}
    \label{appendix:discrete_large_results}
    \vskip 0.15in
    \begin{center}
    \begin{small}
    \begin{sc}
    \begin{adjustbox}{width=\textwidth}
    \begin{tabular}{ccccc}
            \hline
            & NCI1 & NCI109 & PROTEINS & DD \\
             \hline
             \hline
            WL & $85.58 \pm 0.15$ & $84.85 \pm 0.19$ & $\mathbf{76.11 \pm 0.64}$ & $79.45 \pm 0.38$ \\
            WWL & $\mathbf{85.75 \pm 0.25}$ & N/A & $74.28 \pm 0.56$ & $79.69 \pm 0.50$ \\
            PWL-H0 & $85.34 \pm 0.14$ & $84.78 \pm 0.15$ & $75.31 \pm 0.73$ & $79.34 \pm 0.46$ \\
            PWL-H1 & $85.46 \pm 0.16$ & $\mathbf{84.96 \pm 0.34}$ & $75.27 \pm 0.38$ & $78.66 \pm 0.32$ \\
            RetGK-I,II & $84.50 \pm 0.20$ & N/A & $75.80 \pm 0.60$ & $\mathbf{81.60} \pm \mathbf{0.30}$ \\
            \hline
            \hline
            PSCN & $78.59 \pm 1.89$ & N/A & $75.89 \pm 2.76$ & $77.12 \pm 2.41$ \\
            DGCNN & $74.44 \pm 0.47$ & N/A & $75.54 \pm 0.94$ & $79.37 \pm 0.94$ \\
            GCKN-subtree-sup & $82.00 \pm 1.20$ & N/A & $\mathbf{76.20} \pm \mathbf{2.50}$ & N/A \\
            Perslay & $73.50 \pm 0.30$ & $69.50 \pm 0.30$ & $74.80 \pm 0.30$ & N/A \\
            GIN & $82.70 \pm 1.70$ & N/A & $\mathbf{76.20} \pm \mathbf{2.80}$ & $75.30 \pm 2.90$ \\
            GNTK & $84.20 \pm 1.50$ & N/A & $75.60 \pm 4.20$ & N/A \\
            RWNN & $73.90 \pm 1.30$ & N/A & $74.70 \pm 3.30$ & $77.60 \pm 4.70$ \\
            DM & $83.07 \pm 1.07$ & N/A & $\mathbf{76.19} \pm \mathbf{2.91}$ & N/A \\
            \hline
            \hline
            \textbf{PWLR-H0} & $75.66 \pm 0.30$ & $75.52 \pm 0.20$ & $73.24 \pm 0.53$ & $77.22 \pm 0.47$ \\
            $(k_1,k_2)$ & (0,1) & (0,1) & (26,0) & (23,3) \\
            \hline
            \textbf{PWLR-H1} & $74.50 \pm 0.42$ & $75.12 \pm 0.24$ & $72.78 \pm 0.35$ & $76.40 \pm 0.40$ \\
            $(k_1,k_2)$ & (2,0) & (2,0) & (0,5) & (27,0)  \\
            \hline
            \textbf{PWLR-H0+H1} & $77.15 \pm 0.23$ & $76.70 \pm 0.34$ & $73.57 \pm 0.41$ & $76.57 \pm 0.53$ \\
            $(k_1,k_2)$ & (0,1) & (0,1) & (23,4) & (25,3) \\
            \hline
            \hline
            \textbf{PWLR-OPT-H0} & $75.99 \pm 0.23$ & $75.41 \pm 0.18$ & $74.10 \pm 0.59$ & $76.44 \pm 0.46$ \\
            $(k_1,k_2)$ & (15,10) & (6,6) & (0,0) & (11,0) \\
            \hline
            \textbf{PWLR-OPT-H1} & $72.76 \pm 0.33$ & $73.60 \pm 0.25$ & $73.22 \pm 0.39$ & $76.68 \pm 0.30$ \\
            $(k_1,k_2)$ & (2,0) & (2,0) & (29,0) & (0,5) \\
            \hline
            \textbf{PWLR-OPT-H0+H1} & $79.05 \pm 0.40$ & $77.54 \pm 0.19$ & $74.46 \pm 0.38$ & $76.74 \pm 0.38$ \\
            $(k_1,k_2)$ & (2,0) & (1,0) & (0,0) & (9,0) \\
            \hline
    \end{tabular}
    \end{adjustbox}
    \end{sc}
    \end{small}
    \end{center}

    \caption{Dimensions of representations processed in the Random Forest Classifier for classifying large graph data sets with discrete labels. Cells notated as - indicate graph data set which do not have the prescribed node features. The variable $h$ denotes the number of WL iteration procedures used for obtaining the representations. The asterix $*$ denotes that the obtained graph representations are subject to changes based on the choice of training data sets.}
    \label{appendix:large_discrete_dimension}
    \vskip 0.15in
    \begin{center}
    \begin{small}
    \begin{sc}
    \begin{tabular}{ccccc}
        \hline
        data sets & NCI1 & NCI109 & DD & PROTEINS \\
        \hline
        \hline
        Average \# nodes & 29.87 & 29.68 & 284.32 & 39.06 \\
        Average \# edges & 32.30 & 32.13 & 715.66 & 72.82 \\
        \# Discrete Node Labels & 22 & 19 & 82 & 3 \\
        \# Graphs & 4110 & 4127 & 1178 & 1113 \\
        \hline
        \hline
        Graph Kernels (10-fold CV)* & 3699* & 3714* & 1060* & 996* \\
        \hline
        WL, PWL-H0 ($h=1$) & 288 & 275 & 253,231 & 299 \\
        WL, PWL-H0 ($h=10$) & 530,723 & 534,654 & $\geq 10^6$ & 329,035 \\
        WL, PWL-H0 ($h=20$) & $\geq 10^6$ & $\geq 10^6$ & $\geq 10^6$ & 721,222 \\
        \hline
        \hline
        \textbf{PWLR-H0} (Any $k_1$,$k_2$) & \textbf{111} & \textbf{111} & 5748 & 620 \\
        \textbf{PWLR-H1} (Any $k_1$,$k_2$ )& \textbf{18} & \textbf{18} & 8521 & 539 \\
        \textbf{PWLR-OPT} (Any $k_1$,$k_2$) & \textbf{10} & \textbf{12} & \textbf{148} & \textbf{74} \\
        \hline
    \end{tabular}
    \end{sc}
    \end{small}
    \end{center}
    \vskip -0.1in
\end{table}

\begin{figure}[!htbp]
    \centering
    \subfloat[NCI1 data set]{\includegraphics[width=0.36\textwidth]{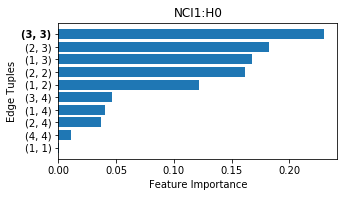} \includegraphics[width=0.36\textwidth]{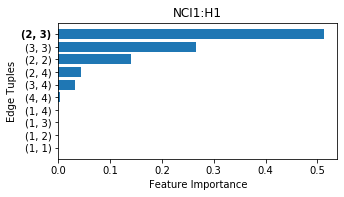}} \quad
    \subfloat[NCI109 data set]{\includegraphics[width=0.36\textwidth]{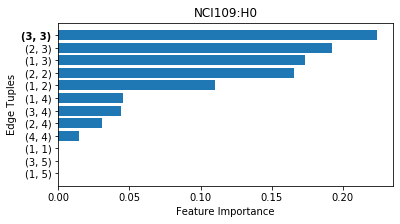} \includegraphics[width=0.36\textwidth]{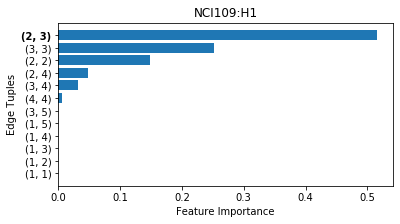}}
    \caption{Exemplary histograms of contributions each component of PWLR-OPT representations make to obtaining graph classification results for NCI data sets}
    \label{appendix:fig:large:NCI}
\end{figure}

\subsection{Graph data sets with Continuous Features}
\label{subsection:graph_continuous}

\begin{table}[!htbp]
    \caption{Classification results obtained from data sets with continuous node labels and edge weights. Cells notated as N/A indicate graph classification techniques which do not report classification results on the given graph data set. The term ``Disc.'' indicates that only the discrete node labels are used for graph classification. The term ``Both.'' indicates that both discrete node labels and continuous node features are used. All classification results other than the proposed method are used from available results obtained from pre-existing publications.}
    \label{appendix:cont_results}
    \vskip 0.15in
    \begin{center}
        \begin{scriptsize}
        \begin{sc}
        \begin{adjustbox}{width=\textwidth}
            \begin{tabular}{cccccccc}
                \hline
                & PROTEINS & \multicolumn{2}{c}{BZR} & BZR-MD & \multicolumn{2}{c}{COX2} & COX2-MD \\
                Features & Both. & Disc. & Both. & & Disc. & Both. & \\
                 \hline
                 \hline
                WWL & $77.91 \pm 0.80$ & \multicolumn{2}{c}{$84.42 \pm 2.03$} & $69.76 \pm 0.94$ & \multicolumn{2}{c}{$78.29 \pm 0.47$} & $\mathbf{76.33 \pm 1.02}$ \\
                HGK-WL & $75.93 \pm 0.17$ & \multicolumn{2}{c}{$78.59 \pm 0.63$} & $68.94 \pm 0.65$ & \multicolumn{2}{c}{$78.13 \pm 0.45$} & $74.61 \pm 1.74$ \\
                HGK-SP & $75.78 \pm 0.17$ & \multicolumn{2}{c}{$76.42 \pm 0.72$} & $66.17 \pm 1.05$ & \multicolumn{2}{c}{$75.45 \pm 1.53$} & $71.83 \pm 1.61$ \\
                GH & $74.78 \pm 0.29$ & \multicolumn{2}{c}{$76.49 \pm 0.99$} & $69.14 \pm 2.08$ & \multicolumn{2}{c}{$76.41 \pm 1.39$} & $66.20 \pm 1.05$  \\
                FGW & $74.55 \pm 2.74$ & \multicolumn{2}{c}{$85.12 \pm 4.15$} & N/A & \multicolumn{2}{c}{$77.23 \pm 4.86$} & N/A \\
                RetGK-I,II & $78.00 \pm 0.30$ & \multicolumn{2}{c}{$87.10 \pm 0.70$} & $62.77 \pm 1.69$ & \multicolumn{2}{c}{$81.40 \pm 0.60$} & $59.47 \pm 1.66$ \\
                FC & $74.54 \pm 0.48$ & \multicolumn{2}{c}{$85.61 \pm 0.59$} & $75.61 \pm 1.13$ & \multicolumn{2}{c}{$81.01 \pm 0.88$} & $73.41 \pm 0.79$ \\
                \hline
                \hline
                PSCN & N/A & \multicolumn{2}{c}{N/A} & $67.00 \pm 9.48$ & \multicolumn{2}{c}{N/A} & $65.33 \pm 7.78$ \\
                DGCNN & N/A & \multicolumn{2}{c}{N/A} & $64.67 \pm 9.32$ & \multicolumn{2}{c}{N/A} & $64.00 \pm 8.86$ \\
                GCKN-subtree-sup & $77.60 \pm 0.40$ & \multicolumn{2}{c}{$86.40 \pm 0.50$} & N/A & \multicolumn{2}{c}{$\mathbf{81.70 \pm 0.70}$} & N/A \\
                Perslay & $74.80 \pm 0.30$ & \multicolumn{2}{c}{N/A} & N/A & \multicolumn{2}{c}{$80.90 \pm 1.00$} & N/A \\
                GIN & N/A & \multicolumn{2}{c}{N/A} & $70.53 \pm 8.00$ & \multicolumn{2}{c}{N/A} & $65.97 \pm 5.70$ \\
                GNTK & $75.70 \pm 0.20$ & \multicolumn{2}{c}{$85.50 \pm 0.80$} & $66.47 \pm 1.20$ & \multicolumn{2}{c}{$79.60 \pm 0.40$} & $64.27 \pm 1.55$ \\
                DM & N/A & \multicolumn{2}{c}{N/A} & $73.55 \pm 5.76$ & \multicolumn{2}{c}{N/A} & $72.28 \pm 9.37$ \\
                \hline
                \hline
                \textbf{PWLR-H0} & 74.62$\pm$0.62 & \textbf{89.39}$\pm$\textbf{0.74} & \textbf{87.81}$\pm$\textbf{0.29} & \text{75.08}$\pm$\text{1.32} & 80.61$\pm$1.11 & 79.26$\pm$0.64 & 71.36$\pm$1.46 \\
                $(k_1,k_2)$ & (2,21) & \textbf{(4,1)} & \textbf{(25,1)} & \text{(0,3)} & (1,13) & (13,1) & (3,0) \\
                \hline
                \textbf{PWLR-H1} & 73.02$\pm$0.50 & \textbf{89.29}$\pm$\textbf{0.66} & \textbf{87.29}$\pm$\textbf{0.65} & \textbf{77.75}$\pm$\textbf{0.98} & 81.02$\pm$0.55 & 79.32$\pm$0.89 & 69.95$\pm$1.26  \\
                $(k_1,k_2)$ & (26,5) & \textbf{(3,4)} & \textbf{(24,10)} & \textbf{(1,4)} & (0,20) & (19,2) & (7,1) \\
                \hline
                \textbf{PWLR-H0+H1} & 73.94$\pm$0.58 & \textbf{88.95}$\pm$\textbf{0.90} & \textbf{87.87}$\pm$\textbf{0.56} & \textbf{76.44}$\pm$\textbf{1.37} & 80.88$\pm$0.54 & 79.34$\pm$0.47 & 71.36$\pm$1.46 \\
                $(k_1,k_2)$ & (7,2) & \textbf{(6,1)} & \textbf{(22,1)} & \textbf{(2,1)} & (0,8) & (7,2) & (3,0) \\
                \hline
                \hline
                \textbf{PWLR-OPT-H0} & $73.73 \pm 0.64$ & \textbf{89.46}$\pm$\textbf{0.55} & \textbf{88.35}$\pm$\textbf{0.50} & 72.47$\pm$1.44 & 79.94$\pm$0.58 & 78.97$\pm$0.69 & 72.10$\pm$2.13 \\
                $(k_1,k_2)$ & (25,0) & \textbf{(2,15)} & \textbf{(29,10)} & (1,1) & (10,0) & (3,15) & (5,1) \\
                \hline
                \textbf{PWLR-OPT-H1} & $73.09 \pm 0.77$ & \textbf{88.93}$\pm$\textbf{0.79} & 86.54$\pm$0.83 & \text{74.18}$\pm$\text{1.18} & 80.22$\pm$0.84 & 77.31$\pm$0.87 & 70.51$\pm$1.86 \\
                $(k_1,k_2)$ & (14,10) & \textbf{(2,0)} & (16,18) & \text{(2,0)} & (15,17) & (21,8) & (26,13) \\
                \hline
                \textbf{PWLR-OPT-H0+H1} & $73.89 \pm 0.38$ & \textbf{89.32}$\pm$\textbf{0.83} &\textbf{88.54}$\pm$\textbf{0.46} & \text{75.55}$\pm$\text{1.08} & 80.97$\pm$1.15 & 78.87$\pm$0.89 &72.97$\pm$1.00 \\
                $(k_1,k_2)$ & (29,22) & \textbf{(11,3)} & \textbf{(2,1)} & \text{(0,11)} & (0,8) & (7,0) & (4,0) \\
                \hline
            \end{tabular}
        \end{adjustbox}
        \end{sc}
        \end{scriptsize}
    \end{center}

    \caption{Dimensions of representations processed in the Random Forest Classifier for classifying graphs with continuous node attributes or edge weights. Cells notated as - indicate graph data set which do not have the prescribed node features. The variable $h$ denotes the number of WL iteration procedures used for obtaining the representations. The asterix $*$ denotes that the obtained graph representations are subject to changes based on the choice of training data sets.}
    \label{appendix:cont_dimension}
    \vskip 0.15in
    \begin{center}
    \begin{small}
    \begin{sc}
    \begin{tabular}{cccccc}
        \hline
        data sets & PROTEINS & BZR & BZR-MD & COX2 & COX2-MD \\
        \hline
        \hline
        Average \# nodes & 39.06 & 35.75 & 21.30 & 41.22 & 26.28 \\
        Average \# edges & 72.82 & 38.36 & 225.06 & 43.45 & 335.12 \\
        \# Discrete Node Labels & 3 & 10 & 8 & 8 & 7 \\
        Dim. continuous node features & 29 & 3 & - & 3 & - \\
        Dim. continuous edge attributes & - & - & 1 & - & 1 \\
        \# Graphs & 1113 & 405 & 306 & 467 & 303 \\
        \hline
        \hline
        Average Dim. all node features & 1249.92 & 464.75 & 170.4 & 453.42 & 183.96 \\
        \hline
        \hline
        Graph Kernels (10-fold CV)* & 996* & 364* & 275* & 420* & 272* \\
        \hline
        \textbf{PWLR-H0} (Any $k_1$,$k_2$) & 620 & \textbf{54} & \textbf{56} & \textbf{33} & \textbf{36} \\
        \textbf{PWLR-H1} (Any $k_1$,$k_2$) & 539 & \textbf{6} & \textbf{5} & \textbf{6} & \textbf{5} \\
        \textbf{PWLR-OPT} (Any $k_1$,$k_2$) & \textbf{74} & \textbf{9} & \textbf{8} & \textbf{9} & \textbf{9} \\
        \hline
    \end{tabular}
    \end{sc}
    \end{small}
    \end{center}
    \vskip -0.1in
\end{table}

\begin{figure}[!htbp]
    \centering
    \subfloat[BZR data set (Discrete labels)]{\includegraphics[width=0.35\textwidth]{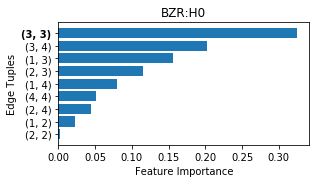} \includegraphics[width=0.35\textwidth]{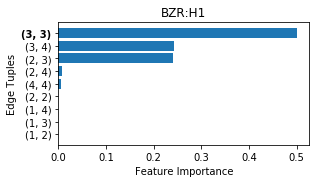}} \quad
    \subfloat[BZR data set (Both discrete and continuous features)]{\includegraphics[width=0.35\textwidth]{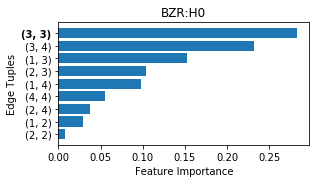} \includegraphics[width=0.35\textwidth]{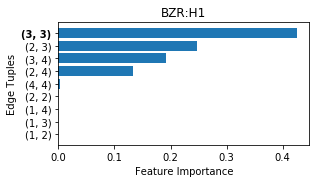}} \quad
    \subfloat[BZR-MD data set]{\includegraphics[width=0.35\textwidth]{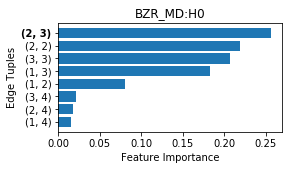} \includegraphics[width=0.35\textwidth]{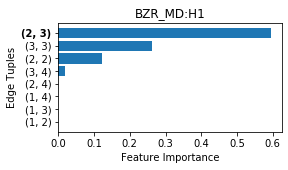}} \quad
    \subfloat[COX2 data set]{\includegraphics[width=0.35\textwidth]{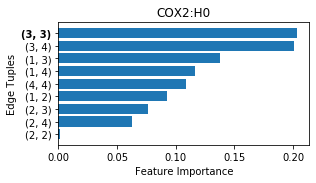} \includegraphics[width=0.35\textwidth]{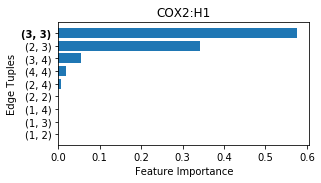}} \quad
    \subfloat[COX2-MD data set]{\includegraphics[width=0.35\textwidth]{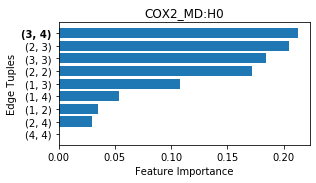} \includegraphics[width=0.35\textwidth]{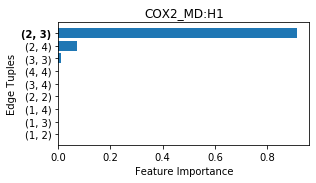}}
    \caption{Exemplary histograms of contributions each component of PWLR-OPT representations make to obtaining graph classification results for BZR, BZR-MD, COX2, and COX2-MD data sets}
    \label{appendix:fig:continuous}
\end{figure}

Table \ref{appendix:cont_results} demonstrates the classification results computed over PROTEINS, BZR, BZR-MD, COX2, and COX2-MD data sets. For preprocessing BZR-MD and COX2-MD data sets, edges whose weights are equal to ``0'' are omitted, as they represent two atoms which are not connected to each other. We also take the reciprocal of the initial 1-dimensional attributes on edges to take consideration of the fact that they represent distances between any two atoms which are bonded together. For classifying BZR and COX2 data sets, we record classification results obtained from solely using discrete labels and using both discrete and continuous node attributes. For classifying PROTEINS data set, we record classification results obtained from using both discrete and continuous node attributes. 

The PWLR embedding framework clearly outperforms state-of-the-art techniques in classifying BZR and BZR-MD data sets by a significant margin. It guarantees comparable performances for classifying COX2 and COX2-MD data sets as well. These enhanced performances are made possible while reducing the embedding dimensions of graph data sets. A detailed description of embedding dimensions and the role each component of reduced PWLR representations contribute to obtain classification scores are provided in Table \ref{appendix:cont_dimension} and Figure \ref{appendix:fig:continuous}. 

\section{Example}
\label{appendix: examples}

In this section, we provide relevant examples to show how PWLR embedding scheme embeds a given finite graph to a real vector space. We explicitly compute how one can utilize persistent homology groups to obtain real embeddings of the finite graph $G$ given as in Figure \ref{fig:graph_ex}. We denote by $v_i$ the $i$-th node of $G$, and by $e_{i,j}$ the edge connecting the nodes $v_i$ and $v_j$.

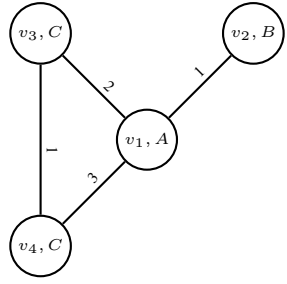
\begin{figure}
\begin{center}
{\tiny
\begin{tikzpicture}[node distance={20mm}, thick, main/.style = {draw, circle}] 
\node[main] (1) {$v_1, A$}; 
\node[main] (2) [above right of=1] {$v_2, B$};
\node[main] (3) [above left of=1] {$v_3, C$};
\node[main] (4) [below left of=1] {$v_4, C$};
\draw (2) -- node[midway, above right, sloped, pos=0.5] {1} (1);
\draw (3) -- node[midway, above right, sloped, pos=0.5] {2} (1);
\draw (4) -- node[midway, above right, sloped, pos=0.5] {3} (1);
\draw (3) -- node[midway, above right, sloped, pos=0.5] {1} (4);
\end{tikzpicture} 
}
\end{center}
\caption{The graph $G$ with initial node labels and edge weights}
\label{fig:graph_ex}
\end{figure}

\subsection{Examples: Normalized WL Procedure}
We illustrate how $M_G X$ corresponds to a single implementation of the normalized WL procedure. The adjacency matrix, its normalized matrix, and the matrix of concatenated node labels is given by
\begin{equation}
    A := \begin{pmatrix} 
    0 & 1 & 2 & 3 \\ 
    1 & 0 & 0 & 0 \\
    2 & 0 & 0 & 1 \\
    3 & 0 & 1 & 0
    \end{pmatrix}, \; \; 
    M_G := (D+I)^{-1}(A+I) = \begin{pmatrix} 
    \frac{1}{7} & \frac{1}{7} & \frac{2}{7} & \frac{3}{7} \\ 
    \frac{1}{2} & \frac{1}{2} & 0 & 0 \\
    \frac{2}{4} & 0 & \frac{1}{4} & \frac{1}{4} \\
    \frac{3}{5} & 0 & \frac{1}{5} & \frac{1}{5}
    \end{pmatrix}, \; \; 
    X := \begin{pmatrix}
    1 & 0 & 0 \\
    0 & 1 & 0 \\
    0 & 0 & 1 \\
    0 & 0 & 1
    \end{pmatrix}
\end{equation}
where we used the coordinate-wise embedding $A \mapsto (1,0,0)$, $B \mapsto (0,1,0)$, and $C \mapsto (0,0,1)$.

As previously mentioned, we regard $M_G X$ as the normalized WL iteration scheme. For example, the label of node 1 changes from $A$ to $ABCCCCC$ after the WL iteration. In particular, among 7 concatenated labels, $\frac{1}{7}$ of the labels belong to $A$, $\frac{1}{7}$ to $B$, and $\frac{5}{7}$ to $C$. Meanwhile, the entries of the matrix $M_G X$ are given by
\begin{equation}
    M_G X = 
    \begin{pmatrix} 
    \frac{1}{7} & \frac{1}{7} & \frac{2}{7} & \frac{3}{7} \\ 
    \frac{1}{2} & \frac{1}{2} & 0 & 0 \\
    \frac{2}{4} & 0 & \frac{1}{4} & \frac{1}{4} \\
    \frac{3}{5} & 0 & \frac{1}{5} & \frac{1}{5}
    \end{pmatrix} \begin{pmatrix}
    1 & 0 & 0 \\
    0 & 1 & 0 \\
    0 & 0 & 1 \\
    0 & 0 & 1
    \end{pmatrix} 
    = 
    \begin{pmatrix}
    \frac{1}{7} & \frac{1}{7} & \frac{5}{7} \\
    \frac{1}{2} & \frac{1}{2} & 0 \\
    \frac{2}{4} & 0 & \frac{2}{4} \\
    \frac{3}{5} & 0 & \frac{2}{5}
    \end{pmatrix}
\end{equation}
where the first row vector of $M_G X$ corresponds to the ratio of labels obtained after a single WL iteration. To see this, the WL iteration scheme updates the node labels as:
\begin{align}
    \begin{split}
        L(v_1) &= ABCCCCC \\
        L(v_2) &= AB \\
        L(v_3) &= AACC \\
        L(v_4) &= AAACC
    \end{split}
\end{align}
We notice that the ratio of occurrences of each label with respect to the total length of the node label is equal to the entries of the row of the matrix $M_G X$.

\subsection{Examples: Computing PWLR Embeddings}

The transition matrix $M_G$ as well as the matrix of node labels are given by
\begin{equation}
    M_G :=     \begin{pmatrix} 
    \frac{1}{7} & \frac{1}{7} & \frac{2}{7} & \frac{3}{7} \\ 
    \frac{1}{2} & \frac{1}{2} & 0 & 0 \\
    \frac{2}{4} & 0 & \frac{1}{4} & \frac{1}{4} \\
    \frac{3}{5} & 0 & \frac{1}{5} & \frac{1}{5}
    \end{pmatrix}, \; \; X := \begin{pmatrix}
    1 & 0 & 0 \\
    0 & 1 & 0 \\
    0 & 0 & 1 \\
    0 & 0 & 1
    \end{pmatrix}
\end{equation}
where we used the one-hot encoding of discrete labels to represent $A \mapsto (1,0,0)$, $B \mapsto (0,1,0)$, and $C \mapsto (0,0,1)$.

Suppose we updated the node labels by only iterating the RW procedure once, i.e. we set $k_1 = 0$ and $k_2 = 1$. The updated node labels $X^T M_G$ is given by
\begin{equation}
    X^T M_G = 
    \begin{pmatrix}
    1 & 0 & 0 & 0 \\
    0 & 1 & 0 & 0 \\
    0 & 0 & 1 & 1
    \end{pmatrix}
    \begin{pmatrix} 
    \frac{1}{7} & \frac{1}{7} & \frac{2}{7} & \frac{3}{7} \\ 
    \frac{1}{2} & \frac{1}{2} & 0 & 0 \\
    \frac{2}{4} & 0 & \frac{1}{4} & \frac{1}{4} \\
    \frac{3}{5} & 0 & \frac{1}{5} & \frac{1}{5}
    \end{pmatrix}
    = 
    \begin{pmatrix}
    \frac{1}{7} & \frac{1}{7} & \frac{2}{7} & \frac{3}{7} \\
    \frac{1}{2} & \frac{1}{2} & 0 & 0 \\
    \frac{11}{10} & 0 & \frac{9}{20} & \frac{9}{20}
    \end{pmatrix}
\end{equation}
The graph $G$, along with its updated node attributes, is shown in Figure \ref{fig:graph_G_ex2}
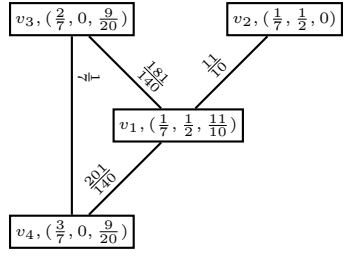
\begin{figure}
\begin{center}
{\tiny
\begin{tikzpicture}[node distance={20mm}, thick, main/.style = {draw, rectangle}] 
\node[main] (1) {$v_1, (\frac{1}{7}, \frac{1}{2}, \frac{11}{10})$}; 
\node[main] (2) [above right of=1] {$v_2, (\frac{1}{7}, \frac{1}{2}, 0)$};
\node[main] (3) [above left of=1] {$v_3, (\frac{2}{7}, 0, \frac{9}{20})$};
\node[main] (4) [below left of=1] {$v_4, (\frac{3}{7}, 0, \frac{9}{20})$};
\draw (2) -- node[midway, above right, sloped, pos=0.75] {$\frac{11}{10}$} (1);
\draw (3) -- node[midway, above right, sloped, pos=0.5] {$\frac{181}{140}$} (1);
\draw (4) -- node[midway, above right, sloped, pos=0.1] {$\frac{201}{140}$} (1);
\draw (3) -- node[midway, above right, sloped, pos=0.15] {$\frac{1}{7}$} (4);
\end{tikzpicture} 
}
\end{center}
\caption{The graph $G$ with node labels updated after a single iteration of RW procedure.}
\label{fig:graph_G_ex2}
\end{figure}
where the heights on edges are computed using the $l_1$ distance between two labels whose nodes are connected by respective edges. The sequence of nested subgraphs
\begin{equation} \label{eq:nested_subgraphs_ex}
    G^{[0]} \subset G^{[1]} \subset G^{[2]} \subset G^{[3]} \subset G^{[4]} = G
\end{equation}
can be drawn as shown in figure \ref{fig:nested_subgraphs}
\begin{figure}[ht]
\centering
\subfloat[$G^{[0]}$]{\small
\begin{tikzpicture}[baseline=(1.base), node distance={20mm}, thick, main/.style = {draw, rectangle}] 
\node[main] (1) {$v_1, (\frac{1}{7}, \frac{1}{2}, \frac{11}{10})$}; 
\node[main] (2) [above right of=1] {$v_2, (\frac{1}{7}, \frac{1}{2}, 0)$};
\node[main] (3) [above left of=1] {$v_3, (\frac{2}{7}, 0, \frac{9}{20})$};
\node[main] (4) [below left of=1] {$v_4, (\frac{3}{7}, 0, \frac{9}{20})$};
\end{tikzpicture} 
}
\hspace{10pt}
\subfloat[$G^{[1]}$]{\small
\begin{tikzpicture}[baseline=(1.base), node distance={20mm}, thick, main/.style = {draw, rectangle}] 
\node[main] (1) {$v_1, (\frac{1}{7}, \frac{1}{2}, \frac{11}{10})$}; 
\node[main] (2) [above right of=1] {$v_2, (\frac{1}{7}, \frac{1}{2}, 0)$};
\node[main] (3) [above left of=1] {$v_3, (\frac{2}{7}, 0, \frac{9}{20})$};
\node[main] (4) [below left of=1] {$v_4, (\frac{3}{7}, 0, \frac{9}{20})$};
\draw (3) -- node[midway, above right, sloped, pos=0.15] {$\frac{1}{7}$} (4);
\end{tikzpicture} 
}

\subfloat[$G^{[2]}$]{\small
\begin{tikzpicture}[baseline=(1.base), node distance={20mm}, thick, main/.style = {draw, rectangle}] 
\node[main] (1) {$v_1, (\frac{1}{7}, \frac{1}{2}, \frac{11}{10})$}; 
\node[main] (2) [above right of=1] {$v_2, (\frac{1}{7}, \frac{1}{2}, 0)$};
\node[main] (3) [above left of=1] {$v_3, (\frac{2}{7}, 0, \frac{9}{20})$};
\node[main] (4) [below left of=1] {$v_4, (\frac{3}{7}, 0, \frac{9}{20})$};
\draw (2) -- node[midway, above right, sloped, pos=0.75] {$\frac{11}{10}$} (1);
\draw (3) -- node[midway, above right, sloped, pos=0.15] {$\frac{1}{7}$} (4);
\end{tikzpicture} 
}
\hspace{10pt}
\subfloat[$G^{[3]}$]{\small
\begin{tikzpicture}[baseline=(1.base), node distance={20mm}, thick, main/.style = {draw, rectangle}] 
\node[main] (1) {$v_1, (\frac{1}{7}, \frac{1}{2}, \frac{11}{10})$}; 
\node[main] (2) [above right of=1] {$v_2, (\frac{1}{7}, \frac{1}{2}, 0)$};
\node[main] (3) [above left of=1] {$v_3, (\frac{2}{7}, 0, \frac{9}{20})$};
\node[main] (4) [below left of=1] {$v_4, (\frac{3}{7}, 0, \frac{9}{20})$};
\draw (2) -- node[midway, above right, sloped, pos=0.75] {$\frac{11}{10}$} (1);
\draw (3) -- node[midway, above right, sloped, pos=0.5] {$\frac{181}{140}$} (1);
\draw (3) -- node[midway, above right, sloped, pos=0.15] {$\frac{1}{7}$} (4);
\end{tikzpicture} 
}
\caption{Figures of nested subgraphs constructed in (\ref{eq:nested_subgraphs_ex})}
\label{fig:nested_subgraphs}
\end{figure}
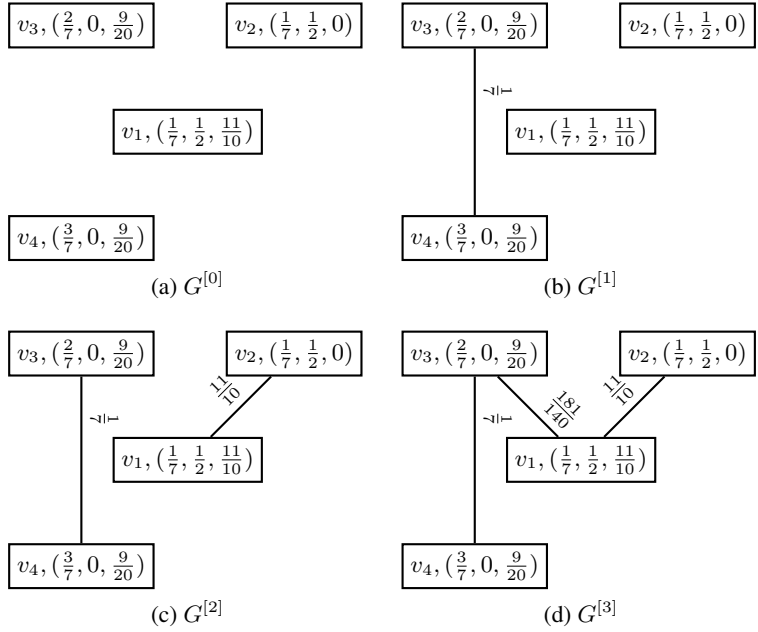

\begin{table}[ht]
    \caption{Number of connected components and cycles obtained from a sequence of nested subgraphs}
    \label{table:PWLR_example}
    \vskip 0.15in
    \begin{center}
    \begin{sc}
    \begin{small}
    \begin{tabular}{c||c|c|c|c|c}
        Features & $G^{[0]}$ & $G^{[1]}$ & $G^{[2]}$ & $G^{[3]}$ & $G$ \\
        \hline
        \hline
        Weights & $0$ & $\frac{1}{7}$ & $\frac{11}{10}$ & $\frac{181}{140}$ & $\frac{201}{140}$ \\
        \hline
        \# Components & 4 & 3 & 2 & 1 & 1 \\
        \hline
        \# Cycles & 0 & 0 & 0 & 0 & 1
    \end{tabular}
    \end{small}
    \end{sc}
    \end{center}
\end{table}

Using these graphs, the number of components and the number of cycles can be summarized as in Table \ref{table:PWLR_example}. The heights on edges that record the variations in the number of components are given by
\begin{equation}
    (h_E(e_{1,2}), h_E(e_{1,3}), h_E(e_{3,4})) = (1.1, 1.292857, 0.142857),
\end{equation}
whereas the height on the edge that records the variations in the number of cycle is given by
\begin{equation}
    (h_E(e_{1,4})) = (1.435714)
\end{equation}
where all the heights are rounded up to 6 decimal places. In this example, suppose we set the bias term $\tau = 0$. The representations of $G$ obtained from sorting these heights are given by
\begin{align}
    \varphi_{H_0}^{[0,1]} &= (0.142857, 1.1, 1.292857) \\
    \varphi_{H_1}^{[0,1]} &= (1.435714)
\end{align}
The sorted tuples of unweighted node degrees representing the edges are given by
\begin{equation}
    \overline{d}_{e_{1,2}} = (1,2), \overline{d}_{e_{1,3}} = (2,3), \overline{d}_{e_{1,4}} = (2,3), \overline{d}_{e_{3,4}} = (2,2)
\end{equation}
We hence obtain $3$-dimensional vectors $\varphi_{H_0, PWLR}^{[0,1]}$ and $\varphi_{H_1, PWLR}^{[0,1]}$ representing the graph $G$ given by
\begin{align}
    \begin{split}
        \varphi_{H_0,Opt}^{[0,1]} &= (1.1,0.142857,1.292857) \\
        \varphi_{H_1,Opt}^{[0,1]} &= (0,0,1.435714)
    \end{split}
\end{align}
where each coordinate of $\varphi_{H_i,PWLR}^{[0,1]}$ represents the sum of weights on edges associated to tuples $(1,2)$, $(2,2)$, and $(2,3)$.

For arbitrary $k_2$ iterations of RW procedures over the node labels of $G$ without any iterations of WL procedures, the sequential representation of $G$ obtained from updated node labels $X^T M_G^{k_2}$ are as shown in Table \ref{tab:example_24}.
\begin{table}[t]
    \caption{Sequential embedding of a finite graph $G$ obtained from node labels of form $X^T M_G^{k_2}$. The heights on edges are computed from taking the $l_1$-distance between labels of two connected nodes}
    \label{tab:example_24}
    \vskip 0.15in
    \begin{center}
    \begin{sc}
    \begin{adjustbox}{width=\textwidth}
    \begin{small}
    \begin{tabular}{c||c|c||c|c}
        $k_2$ & $\varphi^{[0,k_2]}_{H_0}$ & $\varphi^{[0,k_2]}_{H_1}$ & $\varphi^{[0,k_2]}_{H_0,Opt}$ & $\varphi^{[0,k_2]}_{H_1,Opt}$ \\
        Coordinates & Sorted & Sorted & $((1,2),(2,2),(2,3))$ & $((1,2),(2,2),(2,3))$ \\
        \hline
        \hline
        0 & (0, \; 2, \; 2) & (2) & (2, \; 0, \; 2) & (0, \; 0, \; 2) \\
        1 & (0.142857, 1.1, 1.292857) & (1.435714) & (1.1, 0.142857, 1.292857) & (0, \; 0, \; 1.435714) \\
        2 & (0.248980, 0.402398, 0.895) & (0.607806 & (0.895, 0.248980, 0.402398) & (0, \; 0, \; 0.607806) \\
        3 & (0.209344, 0.523549, 1.092750) & (0.732893) & (1.092750, 0.209344, 0.523549) & (0, \; 0, \; 0.732893) \\
        4 & (0.226757, 0.432861, 1.065452) & (0.659618) & (1.065452, 0.226757, 0.432861) & (0, \; 0, \; 0.659618) \\
        5 & (0.219926, 0.459778, 1.103905) & (0.679704) & (1.103905, 0.219926, 0.459778) & (0, \; 0, \; 0.679704) \\
        6 & (0.222978, 0.443228, 1.101213) & (0.666206) & (1.101213, 0.222978, 0.443228) & (0, \; 0, \; 0.666206) \\
        \vdots & \vdots & \vdots & \vdots & \vdots \\
        $\infty$ & (0.222222, 0.444444, 1.111111) & (0.666667) & (1.111111, 0.222222, 0.444444) & (0, \; 0, \;0.666667)
    \end{tabular}
    \end{small}
    \end{adjustbox}
    \end{sc}
    \vskip -0.1in
    \end{center}
\end{table}

\section{Mathematical Background}
\label{appendix: mathematical_background}

In this section, we give a brief outline of the key mathematical ideas required for proving the statements of the theorems in the manuscript. 

\subsection{Markov chains}
\label{appendix: markov_chains}

A classical way to define a stochastic process over a finite graph $G := (V,E)$ is to consider a random walk on $G$. The random walk is characterized by an iterated procedure of starting from a node $v \in V(G)$ and randomly travelling to an adjacent node, including itself. This process over a graph is a particular example of a Markov chain, a stochastic process where the probability of each random choice of walks is only dependent on the present node, independent from previously executed random walks. Interested readers may refer to \cite{IM76, Lo93, La06} for a rigorous treatment of the subject.

Throughout this manuscript, we assume that the probabilities of traveling to any arbitrary adjacent node from the starting point are determined by the positive weights on respective edges. This allows us to model the stochastic process using the normalized adjacency matrix of $G$. Let $M_G$ be the transition matrix given by
\begin{equation}
    M_G := (D+I)^{-1}(A+I)
\end{equation}
The entries of a row of matrix $M_G$ correspond to the probability of traveling to adjacent nodes (including itself) from the given node $v \in V(G)$. Let $P \in \mathbb{R}^{|V|}$ be a coordinate vector $e_i \in \mathbb{R}^{|V|}$, which specifies that the starting point of the random walk is the $i$-th node of $G$. The probability of landing at a node $v \in V(G)$ after taking $k$-steps of random walks can be computed from the matrix $P^T M_G$, where $(\cdot)^T$ is the matrix transpose operator \cite{Lo93}. 

The random walk over a connected finite graph $G$ characterized by the transition matrix $M_G$ satisfies the property that any node can be reached after finitely many steps of random walks, and the probability of arriving at a given node is nonzero after taking sufficiently large number of random walks. One key property of such random walks, is that its large-time is characterized by the left-eigenvector of the transition matrix $M_G$.

\begin{theorem}[Perron-Frobenius Theorem, \cite{IM76}] \label{theorem:perron_frobenius}
Given a connected finite graph $G$, let $M_G$ be the transition matrix characterizing the random walk on $G$. Let $\pi$ be the left eigenvector of $M_G$ with eigenvalue $1$, i.e. $\pi^T = \pi^T M_G$, where $(\cdot)^T$ is the matrix transpose. Let $\nu$ be the right eigenvector of $M_G$ with eigenvalue $1$, i.e. $\nu = M_G \nu$. Suppose the vectors $\pi$ and $\nu$ are normalized such that $\pi^T \nu = 1$. (This constraint implies that $\nu$ is the vector whose entries are all equal to $1$.) Then
\begin{equation}
    \lim_{k \to \infty} M_G^k = \nu \pi^T.
\end{equation}
In particular, given any probability distribution $P$ over the state space $S$,
\begin{equation}
    \lim_{k \to \infty} P M_G^k = \pi^T.
\end{equation}
\end{theorem}

The left eigenvector $\pi$ of $M$ with eigenvalue $1$ is also known as the stationary distribution of the random walk. In fact, the desired left eigenvector $\pi$ is determined by the node degrees of $G$. 

\begin{theorem}[Perron-Frobenius Theorem for Random Walks, \cite{Lo93}]
\label{theorem:stationary_distribution}
Let $G := (V,E)$ be a finite graph. Denote by $d_v$ the weighted degree of a node $v \in V$, i.e. the sum of all weights on edges connected to $v$. Then the left eigenspace of $M_G$ with eigenvalue $1$ is spanned by
\begin{equation}
    \pi_G := \biggl[ \frac{d_v}{\sum_{w \in V(G)} d_w }\biggr]_{v \in V(G)}.
\end{equation}
The right eigenspace of $M_G$ with eigenvalue $1$ is spanned by the vector whose entries are all equal to $1$.
\end{theorem}

Now that we have analyzed the limiting behavior of the sequence $M_G^k$, we focus on determining the convergence rate of the sequence $M_G^k$ to the limit specified in the Perron-Frobenius Theorem. It turns out that for random walks over graphs the sequence $M^k$ geometrically converges to the stationary distribution, a property known as geometric ergodicity.

\begin{theorem} [Geometric Ergodicity (Theorem 5.1, \cite{Lo93})]
Let $M_G$ be the transition matrix which governs the random walk on the graph $G := (V,E)$. Suppose $\mu_2$ is the second largest eigenvalue of $M_G$. Denote by $\pi_G$ the stationary distribution of the transition matrix $M_G$. Then there exists a constant $C > 0$ such that for any probability distribution $P$ on $G$,
\begin{equation}
    \sup_{v \in V} \|P^T M_G^k(v) - \pi_G(v)\|_p < C \mu_2^k.
\end{equation}
\end{theorem}

\subsection{Homology Groups}
\label{appendix: homology_groups}

In the previous section, we observed how constructing a graph $G$ as a 1-dimensional cell complex induce a topological structure over $G$. The topological invariants of a cell complex can be obtained by computing what is called homology groups of $G$. We omit the technical derivation of homology groups, which can be found in Chapter 2 of \cite{Ha02}. 

Because the graph $G$ can be regarded as a 1-dimensional cell complex, it is suffice to consider only the $0$-th and the $1$-st homology groups of $G$. These homology groups (with $\mathbb{R}$-coefficients) are respectively notated as $H_0(G,\mathbb{R})$ and $H_1(G,\mathbb{R})$. Both groups are $\mathbb{R}$-vector spaces whose bases are spanned by the equivalence classes of connected components and cycles of $G$. With abuse of notation, we obtain the following equations.
\begin{align}
    \begin{split}
        \dim_{\mathbb{R}} H_0(G,\mathbb{R}) &= \# \text{ connected components of } G\\
        \dim_{\mathbb{R}} H_1(G,\mathbb{R}) &= \# \text{ cycles of } G
    \end{split}
\end{align}
The dimensions of the $0$-th and the $1$-st homology groups are topological invariants of finite graphs. In other words, two graphs $G$ and $G'$ are not isomorphic if their homology groups are non-isomorphic as real vector spaces. These dimensions can be obtained from the number of nodes and edges of a graph $G$ by using the Euler characteristic formula.
\begin{theorem}[Euler characteristic formula \cite{Ha02}]
\label{theorem: euler_characteristic}
Let $G$ be a 1-dimensional cell complex whose numbers of 0 and 1 cells are finite. Denote by $c_0$ the number of 0-cells, and by $c_1$ the number of 1-cells used for constructing the space $X$. Then the following equality holds for any such $G$:
\begin{equation}
    c_0 - c_1 = \dim_{\mathbb{R}} H_0(G,\mathbb{R}) - \dim_{\mathbb{R}} H_1(G,\mathbb{R})
\end{equation}
The quantity $\chi(G) := c_0 - c_1$ is a topological invariant of $G$, called the Euler characteristic of $G$.
\end{theorem}
In particular, the theorem above implies that
\begin{equation}
    |V| - |E| = \# \text{ connected components of } G - \# \text{ cycles of } G
\end{equation}
for any finite graph $G := (V,E)$.

\subsection{Persistent Homology}
\label{appendix: persistent_homology}

One of the caveats of using homology groups of graphs is that they are too coarse for detecting isomorphic classes of graphs. For example, any two trees $G$ and $G'$ share the same homology groups regardless of the number of nodes and edges. As a measure to impose finer criteria for distinguishing non-isomorphisms, persistent homological techniques allow one to enrich homological data of $G$ by computing a sequence of homology groups of nested subgraphs of $G$. We refer to \cite{ELZ02,Ca09,EH10} for a rigorous treatment of the origins and technical constructions of persistent homology groups associated to $G$.

We define a sequence of nested subgraphs by defining height functions of nodes and edges of $G$, denoted as $h_V:V \to \mathbb{R}$ and $h_E:E \to \mathbb{R}$. Let $\{(w_i^V, w_i^E)\}_{i=1}^N$ be a sequence of real valued upper bounds of the height functions $h_V$ and $h_E$. We construct each graph $G^{[i]}$ as a subgraph of $G$ whose heights on nodes and edges are at most $w_i^V$ and $w_i^E$:
\begin{equation}
    G^{[i]} := \{(v,e) \in (V,E) \; | \; h_V(v) \leq w_i^V, h_E(e) \leq w_i^E\}.
\end{equation}
A family of nested subgraphs of $G$ can now be constructed by utilizing these height functions.
\begin{equation}
    G^{[0]} \subset G^{[1]} \subset G^{[2]} \subset \cdots \subset G^{[N-1]} \subset G^{[N]} = G.
\end{equation}
We note that the filtration of $G$ can be regarded as a procedure for constructing the graph $G$ by consecutively attaching nodes (0-cells) and edges (1-cells) in a manner that is specified by the height functions $h_V$ and $h_E$. The topological structures of nested subgraphs $G^{[i]}$ determine whether the connected components or cycles persist through or vanish as the nodes and edges are sequentially attached. A persistent diagram encapsulates these homological variations by recording the births and deaths of connected components and cycles as nodes and edges are consecutively added \cite{Ca09, EH10}.
\begin{definition}[Persistent Diagrams]
Suppose we have a nested subgraphs of a graph $G := (V,E)$,
\begin{equation}
    G^{[0]} \subset G^{[1]} \subset G^{[2]} \subset \cdots \subset G^{[N-1]} \subset G^{[N]} = G
\end{equation}
determined from height functions on nodes $h_V$ and edges $h_E$. For any $i = 0,1$ and $1 \leq j \leq N$, let $\gamma$ be an $i$-th homology class of $G^{[j]}$ in $H_i(G^{[j]},\mathbb{R})$.
\begin{enumerate}
    \item We say that $\gamma$ is born at $G^{[j]}$ if $\gamma \not\in H_i(G^{[j-1]},\mathbb{R})$.
    \item We say that $\gamma$ dies at $G^{[j+1]}$ if it merges with a distinct homology class $\gamma' \in H_i(G^{[j]},\mathbb{R})$ in the homology group $H_i(G^{[j+1]}, \mathbb{R})$ of $G^{[j+1]}$.
    \item Suppose that an $i$-th homology class $\gamma$ is born at $G^{[j]}$ and dies at $G^{[k]}$. The persistence tuple is a tuple of weights $(w_j,w_k)$ such that the heights on the $i$-cell of $G^{[j]}$ and $G^{[k]}$ are at most $w_j$ and $w_k$. If the $i$-th homology class $\gamma$ is born at $G^{[j]}$ but persists as a homology class of $G$, then we notate the persistence tuple as $(w_j, \infty)$.
    \item The persistent diagram of $G$ is a list of tuples $\{(w_j^i, w_k^i)\}_{i=1}^N$ including multiplicities which records the births and deaths of homology classes of $G^{[i]}$.
\end{enumerate}
\end{definition}
We end this section with a few examples which materializes the computation of persistent homology groups of finite graphs. The persistent WL procedure utilizes height functions induced from node labels updated from applying $h$ iterations of 1-WL procedures \cite{RBB19}.
\begin{align}
\begin{split}
    h_V(v) &= 0 \text{ for any } v \in V \\
    h_E((v_1,v_2)) &= \|L(v_1)-L(v_2)\|_p \text{ for any } v_1, v_2 \in V
\end{split}
\end{align}
In addition, Carriere et al. proposed a novel neural network layer called Perslay \cite{CC20}, where they define the height function using Heat Kernel Signature with diffusion parameter $t$, denoted as $hks_{G,t}$ \cite{HRG14}.
\begin{align}
\begin{split}
    h_V(v) &= hks_{G,t}(v) \text{ for any } v \in V \\
    h_E((v_1,v_2)) &= \max{(hks_{G,t}(v_1), hks_{G,t}(v_2))} \\
    & \text{ for any } v_1, v_2 \in V
\end{split}
\end{align}
Given predetermined height functions $h_V$ and $h_E$, a family of nested subgraphs can be constructed,
\begin{equation}
    G^{[0]} \subset G^{[1]} \subset G^{[2]} \subset \cdots \subset G^{[N-1]} \subset G^{[N]} = G,
\end{equation}
where the graph $G^{[i]}$ is defined from a sequence of real-valued upper bounds $\{(w_i^V, w_i^E)\}_{i=1}^N$ such that
\begin{equation}
    G^{[i]} := \{(v,e) \in (V,E) \; | \; h_V(v) \leq w_i^V, \; h_E(e) \leq w_i^E\}.
\end{equation}

\section{Proofs}
\label{appendix: technical_proofs}

Using the theory of universal covering spaces, Markov chains, and persistent homological techniques, we proceed to prove the statements of Theorem \ref{theorem: three_properties}, Proposition \ref{proposition: normalized_wl}, and Theorem \ref{theorem:stability}.

\subsection{Proof of Proposition \ref{proposition: normalized_wl}}
\label{appendix: proposition: WL}

Before we prove Theorem \ref{theorem: three_properties}, we first observe how the operator $M_G \times (\cdot)$ is equivalent to a single iteration of normalized WL procedure. Given a graph $G := (V,E)$, let $M_G$ be the normalized adjacency matrix defined as 
\begin{equation}
    M_G := (D+I)^{-1}(A+I),
\end{equation}
where $A := \{ \alpha_{v,w} \}_{v,w \in V(G)}$ is the square weighted adjacency matrix, and $D := \{ d_v \}_{v \in V(G)}$ is a diagonal matrix of weighted node degrees. Fix a node $v \in v(G)$ of $G$. The entries of the row of $M_G$ corresponding to the node $v$ are given by 
\begin{equation}
    \left[ \frac{\alpha_{v,v_1}}{d_v + 1}, \frac{\alpha_{v,v_2}}{d_v + 1}, \cdots, \frac{\alpha_{v,v} + 1}{d_v + 1}, \cdots, \frac{\alpha_{v,|V|}}{d_v + 1} \right]
\end{equation}

Let $L:V(G) \to \mathbb{R}^l$ be the function which endows positive labels to nodes of $G$. For each node $v \in V(G)$, we denote by $L(v)_i$ the $i$-th coordinate of its positive label. The node labels define the concatenated matrix of node labels $X \in \mathbb{R}^{l \times |V|}$. The entries of $X$ are given by
\begin{equation}
    X = \left[ L(v)_i \right]_{v \in V(G), 1 \leq i \leq l}
\end{equation}

Let us further suppose that the $l_1$ norm of these node labels are all equal to $1$. That is, for any node $v \in V(G)$, $\sum_{i=1}^l |L(v)_i| = 1$. To prove the proposition, it suffices to show that the entries of $M_G \times X$ are equal to the node labels updated from a single implementation of the normalized WL procedure. Let $\tilde{\alpha}_{v,w}$ be weights on edges defined as
\begin{equation}
    \tilde{\alpha}_{v,w} = \begin{cases}
    \alpha_{v,w} & \text{ if } w \neq v \\
    \alpha_{v,v} + 1 & \text{ otherwise }
    \end{cases}
\end{equation}
The $(v,i)$-th entry of $M_G \times X$ corresponding to the node $v$ can be computed as
\begin{align}
\begin{split}
    &\left[ \frac{\alpha_{v,v_1}}{d_v + 1}, \frac{\alpha_{v,v_2}}{d_v + 1}, \cdots, \frac{\alpha_{v,v} + 1}{d_v + 1}, \cdots, \frac{\alpha_{v,v_|V|}}{d_v + 1} \right] \times \left[ L(v_1)_i, L(v_2)_i, \cdots, L(v)_i, \cdots, L(v_{|V|})_i \right]^T \\
    &= \frac{\alpha_{v,v}+1}{d_v+1} L(v)_i + \sum_{w \in V(G), w \neq v} \frac{\alpha_{v,w}}{d_v+1} L(w)_i = \frac{\sum_{w \in V(G), w \neq v} \alpha_{v,w} L(w)_i}{d_v+1} \\
    &= \frac{\sum_{w \in V(G)} \tilde{\alpha}_{v,w} L(w)_i}{\sum_{w \in V(G)} \tilde{\alpha}_{v,w} \|L(w)\|_1}.
\end{split}
\end{align}
The last equality follows from the constriant that $\sum_{i=1}^l \|L(w)_i\| = 1$, because
\begin{equation}
    d_v + 1 = \sum_{w \in V(G)} \tilde{\alpha}_{v,w} = \sum_{w \in V(G)} \tilde{\alpha_{v,w}} \times 1 = \sum_{w \in V(G)} \tilde{\alpha_{v,w}} \|L(w)\|_1.
\end{equation}

\subsection{Proof of Theorem \ref{theorem: three_properties}}
\label{appendix: theorem: three_properties}

In this subsection, we demonstrate how the normalized WL procedure, RW procedure, and persistent homology groups represent local topological properties, node degrees, and global topological invariants of a finite graph $G$. We recall the statement of the theorem.
\begin{theorem}
Given a finite undirected graph $G = (V,E)$ without self-loops, each component of the PWLR embedding procedure incorporates the following properties of $G$.
\begin{enumerate}
    \item The component $M_G^{k_1} \times (\cdot)$ (\textbf{WL}) incorporates local topological properties of $G$ by representing depth $k_1$ unfolding trees with fixed vertices.
    \item The component $\cdot \times M_G^{k_2}$ (\textbf{M}) incorporates node degrees with local topological properties of $G$.
    \item The component $\varphi(\cdot)$ (\textbf{P}) incorporates global topological invariants of $G$, namely the connected components and cycles of $G$. 
\end{enumerate}
\end{theorem}

\begin{proof}
\medskip
\textbf{Local topological properties}
\medskip

It is a well-known fact that the label of a node $v \in V(G)$ obtained from $k$ iterations of WL procedure represents the depth $k$-unfolding tree at $v$ \cite{SS11}. The correspondence originates from the fact that the updated label at node $v$ is the concatenation of all initial node labels (or component-wise addition of one-hot encoded vectors) of the depth $k$-unfolding tree at $v$. Proposition \ref{proposition: normalized_wl} suggests that the node features obtained from the operation $M_G^{k_1} \times (\cdot)$ can be obtained from linear aggregations of initial node labels of depth $k$ unfolding trees.  This demonstrates that the operation $M_G^{k_1} \times (\cdot)$ represents local topological properties of $G$.

\medskip
\textbf{Node degrees}
\medskip

Let $X$ be the concatenated matrix of node labels of $G$. By the Perron-Frobenius Theorem,
\begin{equation}
    \lim_{k \to \infty} M_G^k = \nu \pi_G^T
\end{equation}
where $\nu$ is the vector whose entries are all equal to $1$, and $\pi_G$ is the vector whose entries are given by
\begin{equation}
    \pi_G := \left[ \frac{d_v}{\sum_{w \in V(G)} d_w} \right]_{v \in V(G)}.
\end{equation}
Because $X$ is a matrix of node labels, there exists constants $C_1, C_2, \cdots, C_l > 0$ and probability distributions $P_1, P_2, \cdots, P_l$ over the set of nodes $V$ such that
\begin{equation}
    X = \begin{pmatrix}
    \vert & \vert & \cdots & \vert \\
    C_1 P_1 & C_2 P_2 & \cdots & C_l P_l \\
    \vert & \vert & \cdots & \vert
    \end{pmatrix}
\end{equation}
Let $\Pi_G$ be the matrix of concatenated stationary distributions defined as
\begin{equation}
    \Pi_G := \begin{pmatrix}
    \vert & \vert & \cdots & \vert \\
    C_1 \pi_G & C_2 \pi_G & \cdots & C_l \pi_G \\
    \vert & \vert & \cdots & \vert
    \end{pmatrix}
\end{equation}
Then the geometric ergodicity of Markov chains implies
\begin{align}
    \begin{split}
        \sup_{v \in V} \|X^T M_G^k(v) - \Pi_G(v)\|_p &\leq \sum_{k=1}^l \sup_{v \in V} C_k \|P_l^T M_G^k(v) - \pi_G(v)\|_p < C \sum_{k=1}^l C_k \mu_2^k.
    \end{split}
\end{align}
where $\mu_2$ is the second largest eigenvalue of $M_G$. Because $M_G$ is a stochastic matrix (Perron-Frobenius matrix), it holds that $\mu_2 < 1$. Therefore, the operator $\cdot \times M_G^{k_2}$ determines the extent of incorporating node degrees $G$ with a given set of node labels of $G$.

\medskip
\textbf{Global topological invariants}
\medskip

One may observe from the construction of persistent diagrams that the set of homology classes which persists throughout the sequence of nested subgraphs of $G$ spans the homology groups of $G$. We prove the theorem using the analogous argument shown in \cite{RBB19}. The height function on the set of nodes $h_V : V \to \mathbb{R}$ is the zero function, whereas the height function on the set of edges $h_E: E \to \mathbb{R}$ is given by the $l_p$ distance between the labels of two adjacent nodes connected by an edge. Let us recall the Euler characteristic formula for each subgraph $G^{[i]} := (V,E^i)$.
\begin{equation}
    |V| - |E^i| = \# \text{ connected components of } G^{[i]} - \# \text{ cycles of } G^{[i]}.
\end{equation}
Note that $|V|$ is a constant, and $|E^i| < |E^{i+1}|$ for all $i$'s. Thus, the Euler characteristic formula implies that as $i$ increases, either the number of connected components decreases or the number of cycles increases. The corresponding persistent tuples are thus of form $(0,w^V)$ or $(w^E,\infty)$. One may represent the tuple $(0,w^V)$ as the weight $w^V$, and the tuple $(w^E, \infty)$ as the weight $w^E$. We recall that the terminal entry of the nested subgraph is equal to $G$. Therefore, as $i$ increases, the number of connected components of $G^{[i]}$ decreases to the number of connected components of $G$, and the number of cycles of $G^{[i]}$ increases to the number of cycles of $G$. In other words, the number of $i$-th dimensional persistent tuples, including multiplicities, is equal to the number of nodes of $G$ and the number of cycles of $G$. We can hence conclude that the Euclidean embedding obtained from this persistent homological technique incorporates global topological invariants of $G$.
\end{proof}

\subsection{Proof of Theorem \ref{theorem:stability}}
\label{appendix: theorem: stability}

In this section, we prove that utilizing Markov chains over finite connected graphs guarantees the following two properties of the proposed graph representation:
\begin{enumerate}
    \item Numerical indications on the incorporation of local topological features and node degree data.
    \item Stability of the proposed embedding with respect to graph perturbation.
\end{enumerate}

The key property of Markov chains used to ensure these merits is that the distance between two stationary distributions of two Markov chains with sufficiently small perturbations is also sufficiently small.

\begin{definition}
Let $G$, $G'$ be two graphs with the same number of nodes $|V|$. Denote by $M_G$ and $M_{G'}$ the Markov chain on $G$ and $G'$ corresponding to weighted random walks on these graphs. Let $\epsilon \in \mathbb{R}^{|V| \times |V|}$ be a $|V| \times |V|$ matrix such that
\begin{equation}
    M_{G'} = M_G + \epsilon.
\end{equation}
We then say the graph $G'$ is perturbed by $\epsilon$ from $G$.
\end{definition}

\begin{theorem} [Perturbation Theory (Theorem 1, Section 4, \cite{Sc68})]
Let $M_G, M_{G'}$ be two transition matrices characterizing random walks over $G$ and $G'$. Let $\epsilon$ be a square matrix defined as $\epsilon := M_{G'} - M_G$. Let $Z_G$ be the fundamental matrix of the Markov chain $M_G$, defined as $Z_G := (I - M_G)^{-1}$. Suppose that the p-matrix norm of $\epsilon Z_G$ is at most $1$, i.e.
\begin{equation}
    \| \epsilon Z_G \|_p < 1.
\end{equation}
Denote by $\pi_G$ and $\pi_{G'}$ the stationary distributions of $M_G$ and $M_{G'}$. Then
\begin{equation}
    \|\pi_{G'} - \pi_{G}\|_p < \|\pi_G\|_p \|Z_G\|_p \|\epsilon\|_p.
\end{equation}
\end{theorem}

Using these theorems, we prove the stability of PWLR embedding scheme of form (\ref{eq:first_representation}) with respect to graph perturbations. We restate the statement of Theorem \ref{theorem:stability} in a more rigorous manner.
\begin{theorem} [Stability]
\label{theorem:appendix_stability}
Let $G,G'$ be two connected graphs with the same number of nodes $V = V(G) = V(G')$, whose associated Markov chains governing the random walks over these graphs are given by $M_G$ and $M_{G'}$. Let $Z_G$ be the fundamental matrix of the Markov chain $M_G$, defined as $Z_G := (I - M_G)^{-1}$. Denote by $\pi_G$ and $\pi_{G'}$ the stationary distributions of $M_G$ and $M_{G'}$. Denote by $\mu_{2,G}$ and $\mu_{2,G'}$ the second largest eigenvalues of $M_G$ and $M_{G'}$. 

Let $L_G: V(G) \to (\mathbb{R}_{\geq 0})^l$ and $L_{G'}: V(G') \to (\mathbb{R}_{\geq 0})^l$ be node labels of $G$ and $G'$ for some number $l$. Denote by $\|L_G\|_1$ and $\|L_{G'}\|_1$ the sum of $l_1$-norms of all node labels of $G$ and $G'$, i.e. 
\begin{align}
    \begin{split}
        \|L_G\|_1 &:= \sum_{v \in V(G)} \|L_G(v)\|_1 \\
        \|L_{G'}\|_1 &:= \sum_{w \in V(G')} \|L_{G'}(w)\|_1
    \end{split}
\end{align}
Suppose that there exists a $|V| \times |V|$ real matrix 
\begin{equation}
    \epsilon := M_{G'} - M_G
\end{equation}
which satisfies
\begin{equation}
    \|\epsilon Z_G\|_p < 1.
\end{equation}
Then there exist constants $C_1$ and $C_2$ such that for any $p \geq 1$,
\begin{small}
\begin{align}
\begin{split}
    &\|\varphi_{H_0}^{[k_1,k_2]}(G) - \varphi_{H_0}^{[k_1',k_2']}(G')\|_p < |V|^{\frac{1}{p}} (\|L_G\|_1 + \|L_{G'}\|_1) (C_1 \mu_{2,G}^{k_2} + 2\|\pi_G\|_p\|Z_G\|_p\|\epsilon\|_p + C_2 \mu_{2,G'}^{k_2'})
\end{split}
\end{align}
\end{small}
Furthermore, if the number of edges of $G$ and $G'$ are equal, i.e. $|E| = |E(G)| = |E(G')|$, then
\begin{scriptsize}
\begin{align}
\begin{split}
    &\|\varphi_{H_1}^{[k_1,k_2]}(G) - \varphi_{H_1}^{[k_1',k_2']}(G')\|_p < (|E|-|V|+1)^{\frac{1}{p}} (\|L_G\|_1 + \|L_{G'}\|_1) (C_1 \mu_{2,G}^{k_2} + 2\|\pi_G\|_p\|Z_G\|_p\|\epsilon\|_p + C_2 \mu_{2,G'}^{k_2'})
\end{split}
\end{align}
\end{scriptsize}
\end{theorem}
\begin{proof}

Let $V(G) := \{v_1, \cdots, v_n\}$ be the set of nodes of $G$, and let $V(G') := \{w_1, \cdots, w_n\}$ be the set of nodes of $G'$. 

\medskip
\textbf{Probability Distribution}
\medskip

Let $p_i: \mathbb{R}^l \to \mathbb{R}$ be the projection map to the $i$-th coordinate of an $l$-dimensional real vector. We first prove the statement of the theorem for the case where each component of the node labels $L_G$ and $L_{G'}$ defines probability distributions over $G$ and $G'$, i.e. for any $1 \leq i \leq l$,
\begin{equation}
    \sum_{v \in V(G)} \|p_i(L_G(v))\| = \sum_{w \in V(G')} \|p_i(L_{G'}(w))\| = 1
\end{equation}

Without loss of generality, fix an integer $i$ which satisfies $1 \leq i \leq l$. Denote the initial probability distribution over the nodes of graph $G$ and $G'$ induced from the $i$-th coordinate of node labels $L_G$ and $L_{G'}$ as $P_G$ and $P_{G'}$. Note that for each node $v \in V(G)$ and $w \in V(G')$, we have
\begin{align}
    \begin{split}
        P_G(v) &:= p_i(L_G(v)) \\
        P_{G'}(w) &:= p_i(L_{G'}(w))
    \end{split}
\end{align}

We show that for any two pairs of vertices $v_i, v_j \in V(G)$ and $w_i, w_j \in V(G')$, and any two positive integers $k_2$ and $k_2'$,
\begin{scriptsize}
\begin{equation} \label{equation:theorem}
    |\|P_G^T M_G^{k_2} (v_1) - P_G^T M_G^{k_2} (v_2)\|_p - \|P_{G'}^T M_{G'}^{k_2'}(w_1) - P_{G'}^T M_{G'}^{k_2'}(w_2)\|_p | < C_1 \mu_{2,G}^{k_2} + 2\|\pi_G\|_p\|Z_G\|_p\|\epsilon\|_p + C_2 \mu_{2,G'}^{k_2'}.
\end{equation}
\end{scriptsize}
By the triangle inequality, we have
\begin{tiny}
\begin{align} \label{equation:triangle_inequality}
    \begin{split}
        \|P_G^T M_G^{k_2} (v_1) - P_G^T M_G^{k_2} (v_2)\|_p \leq &\|P_G^T M_G^{k_2} (v_1) - \pi_G (v_1)\|_p + \|\pi_G (v_1) - \pi_G (v_2)\|_p + \|\pi_G (v_2) - P_G^T M_G^{k_2} (v_2)\|_p \\
        \|P_{G'}^T M_{G'}^{k_2} (w_1) - P_{G'}^T M_{G'}^{k_2} (w_2)\|_p \leq &\|P_{G'}^T M_{G'}^{k_2} (w_1) - \pi_{G'} (w_1)\|_p + \|\pi_{G'} (w_1) - \pi_{G'} (w_2)\|_p + \|\pi_{G'} (w_2) - P_{G'}^T M_{G'}^{k_2} (w_2)\|_p \\
        \|\pi_{G}(v_1) - \pi_G(v_2)\|_p \leq &\|\pi_{G}(v_1) - \pi_{G'}(w_1)\|_p + \|\pi_{G'}(w_1) - \pi_{G'}(w_2)\|_p + \|\pi_{G'}(w_2) - \pi_{G}(v_2)\|_p
    \end{split}
\end{align}
\end{tiny}
Observe that
\begin{align} \label{equation:theorem_sub1}
\begin{split}
    & |\|P_G^T M_G^{k_2} (v_1) - P_G^T M_G^{k_2} (v_2)\|_p - \|P_{G'}^T M_{G'}^{k_2'}(w_1) - P_{G'}^T M_{G'}^{k_2'}(w_2)\|_p | \\
    &< |\|P_G^T M_G^{k_2} (v_1) - P_G^T M_G^{k_2} (v_2)\|_p - \|\pi_{G}(v_1) - \pi_G(v_2)\|_p | \\
    &+ |\|\pi_{G}(v_1) - \pi_G(v_2)\|_p - \|\pi_{G'}(w_1) - \pi_{G'}(w_2)\|_p | \\
    &+ |\|\pi_{G'}(w_1) - \pi_{G'}(w_2)\|_p - \|P_{G'}^T M_{G'}^{k_2'}(w_1) - P_{G'}^T M_{G'}^{k_2'}(w_2)\|_p |
\end{split}
\end{align}
Without loss of generality, assume that 
\begin{align}
    \begin{split}
        \|P_{G}^T M_{G}^{k_2} (v_1) - P_{G}^T M_{G}^{k_2} (v_2)\|_p &> \|\pi_{G}(v_1) - \pi_G(v_2)\|_p \\
    \|\pi_{G}(v_1) - \pi_G(v_2)\|_p &> \|\pi_{G'}(w_1) - \pi_{G'}(w_2)\|_p \\
    \|P_{G'}^T M_{G'}^{k_2'}(w_1) - P_{G'}^T M_{G'}^{k_2'}(w_2)\|_p &> \|\pi_{G'}(w_1) - \pi_{G'}(w_2)\|_p
    \end{split}
\end{align}
Substituting (\ref{equation:triangle_inequality}) to (\ref{equation:theorem_sub1}) gives
\begin{tiny}
\begin{align} \label{equation:theorem_sub2}
    \begin{split}
        &\|P_G^T M_G^{k_2} (v_1) - P_G^T M_G^{k_2} (v_2)\|_p - \|\pi_G (v_1) - \pi_G (v_2)\|_p \leq \|P_G^T M_G^{k_2} (v_1) - \pi_G (v_1)\|_p + \|\pi_G (v_2) - P_G^T M_G^{k_2} (v_2)\|_p \\
        &\|\pi_{G}(v_1) - \pi_G(v_2)\|_p - \|\pi_{G'}(w_1) - \pi_{G'}(w_2)\|_p \leq \|\pi_{G}(v_1) - \pi_{G'}(w_1)\|_p + \|\pi_{G'}(w_2) - \pi_{G}(v_2)\|_p \\
        &\|P_{G'}^T M_{G'}^{k_2} (w_1) - P_{G'}^T M_{G'}^{k_2} (w_2)\|_p - \|\pi_{G'} (w_1) - \pi_{G'} (w_2)\|_p \leq \|P_{G'}^T M_{G'}^{k_2} (w_1) - \pi_{G'} (w_1)\|_p + \|\pi_{G'} (w_2) - P_{G'}^T M_{G'}^{k_2} (w_2)\|_p
    \end{split}
\end{align}
\end{tiny}
Observe that the first two terms of (\ref{equation:theorem_sub2}) are bounded above by $C_1 \mu_{2,G}^{k_2}$ for some constant $C_1$ by the geometric ergodicity of finite state-space Markov chains. Likewise, the last two terms of (\ref{equation:theorem_sub2}) are bounded above by $C_2 \mu_{2,G'}^{k_2'}$ for some constant $C_2$. The remaining two terms in the middle are bounded above by $2 \|\pi_G\|_p \|Z_G\|_p \|\epsilon\|_p$ by the perturbation theory of finite state-space Markov chains. Therefore, we obtain that
\begin{small}
\begin{equation*}
    |\|P_G^T M_G^{k_2} (v_1) - P_G^T M_G^{k_2} (v_2)\|_p - \|P_{G'}^T M_{G'}^{k_2'}(w_1) - P_{G'}^T M_{G'}^{k_2'}(w_2)\|_p | < C_1 \mu_{2,G}^{k_2} + 2\|\pi_G\|_p\|Z_G\|_p\|\epsilon\|_p + C_2 \mu_{2,G'}^{k_2'}.
\end{equation*}
\end{small}

\medskip
\textbf{Node Labels}
\medskip

Now, denote by $X_G$ and $X_{G'}$ the $|V| \times l$-dimensional real matrix consisting of concatenated node labels of $G$ and $G'$. Denote by $(X_G)_i$ and $(X_{G'})_i$ the $i$-th component of $X_G$ and $X_{G'}$. Notice that $\frac{(X_G)_i}{|(X_G)_i|_1}$ and $\frac{(X_{G'})_i}{|(X_{G'})_i|_1}$ define probability distributions over the state spaces $G$ and $G'$. By linearity, we obtain from (\ref{equation:theorem}) that
\begin{align*}
    &|\|(X_G)_i^T M_G^{k_2} (v_1) - (X_G)_i^T M_G^{k_2} (v_2)\|_p - \|(X_{G'})_i^T M_{G'}^{k_2'}(w_1) - (X_{G'})_i^T M_{G'}^{k_2'}(w_2)\|_p | \\
    &< (\|(X_G)_i\|_1 + \|(X_{G'})_i\|_1)(C_1 \mu_{2,G}^{k_2} + 2\|\pi_G\|_p\|Z_G\|_p\|\epsilon\|_p + C_2 \mu_{2,G'}^{k_2'}).
\end{align*}
Because $X_G$ and $X_{G'}$ are concatenations of column vectors $(X_G)_i$ and $(X_{G'})_i$ for $1 \leq i \leq l$,
\begin{align*}
    &|\|(X_G)_i^T M_G^{k_2} (v_1) - (X_G)_i^T M_G^{k_2} (v_2)\|_p - \|(X_{G'})_i^T M_{G'}^{k_2'}(w_1) - (X_{G'})_i^T M_{G'}^{k_2'}(w_2)\|_p | \\
    &< (\|(X_G)_i\|_1 + \|(X_{G'})_i\|_1)(C_1 \mu_{2,G}^{k_2} + 2\|\pi_G\|_p\|Z_G\|_p\|\epsilon\|_p + C_2 \mu_{2,G'}^{k_2'}).
\end{align*}
Therefore, for any number $p \geq 1$, we have
\begin{align*}
    &|\|X_G^T M_G^{k_2} (v_1) - X_G^T M_G^{k_2} (v_2)\|_p - \|X_{G'}^T M_{G'}^{k_2'}(w_1) - X_{G'}^T M_{G'}^{k_2'}(w_2)\|_p | \\
    &= \left( \sum_{i=1}^l |\|(X_G)_i^T M_G^{k_2} (v_1) - (X_G)_i^T M_G^{k_2} (v_2)\|_p - \|(X_{G'})_i^T M_{G'}^{k_2'}(w_1) - (X_{G'})_i^T M_{G'}^{k_2'}(w_2)\|_p |^p \right)^{\frac{1}{p}} \\
    &< \left( \sum_{i=1}^l (\|(X_G)_i\|_1 + \|(X_{G'})_i\|_1)^p \right)^{\frac{1}{p}}
    (C_1 \mu_{2,G}^{k_2} + 2\|\pi_G\|_p\|Z_G\|_p\|\epsilon\|_p + C_2 \mu_{2,G'}^{k_2'}) \\
    &< (\sum_{i=1}^l \|(X_G)_i\|_1 + \|(X_{G'})_i\|_1 )
    (C_1 \mu_{2,G}^{k_2} + 2\|\pi_G\|_p\|Z_G\|_p\|\epsilon\|_p + C_2 \mu_{2,G'}^{k_2'}). \\
    &< (\|L_G\|_1 + \|L_{G'}\|_1)
    (C_1 \mu_{2,G}^{k_2} + 2\|\pi_G\|_p\|Z_G\|_p\|\epsilon\|_p + C_2 \mu_{2,G'}^{k_2'}).
\end{align*}
In short, for any $p \geq 1$,
\begin{align} \label{equation:theorem_sub3}
\begin{split}
    &|\|X_G^T M_G^{k_2} (v_1) - X_G^T M_G^{k_2} (v_2)\|_p - \|X_{G'}^T M_{G'}^{k_2'}(w_1) - X_{G'}^T M_{G'}^{k_2'}(w_2)\|_p | \\
    &< (\|L_G\|_1 + \|L_{G'}\|_1)
    (C_1 \mu_{2,G}^{k_2} + 2\|\pi_G\|_p\|Z_G\|_p\|\epsilon\|_p + C_2 \mu_{2,G'}^{k_2'}).
\end{split}
\end{align}

\medskip
\textbf{Proof of the theorem}
\medskip

The statement of the theorem follows immediately from (\ref{equation:theorem_sub3}). We first note that $M_G^{k_1}X$ still defines a node label on $G$ whose $l_1$-norm is equal to that of $X$. In other words, for any $k_1> 0$,
\begin{equation}
    \|L_{M_G^{k_1,0} G}\|_1 = \|L_{G}\|_1.
\end{equation}
Recall that both representations $\varphi_{H_0}$ and $\varphi_{H_1}$ are lists of sorted heights over the edges of $G$ and $G'$. By Euler's characteristic formula, the dimensions of $\varphi_{H_0}(G)$ and $\varphi_{H_0}(G')$ are equal to $|V(G)|-1$ and $|V(G')|-1$, whereas the dimensions of $\varphi_{H_1}(G)$ and $\varphi_{H_1}(G')$ are equal to $|E(G)| - |V(G)| + 1$ and $|E(G')| - |V(G')| + 1$. Observe that the bias term $\tau$ from (\ref{eq:first_representation}) does not contribute to increment of $L^p$ distance between the representations $\varphi_{H_0}(G)$ and $\varphi_{H_0}(G')$. This is because the bias term $\tau$ is uniformly added to all heights over edges. Thus for any $p \geq 1$, 
\begin{align}
\begin{split}
    &|\varphi_{H_0}^{[k_1,k_2]}(G) - \varphi_{H_0}^{[k_1',k_2']}(G')|_p  \\
    \leq &\biggl( \sum_{n=1}^{|V|-1} \max_{\substack{v_1, v_2 \in V(G) \\ w_1, w_2 \in V(G')}} \biggl| \|X_G^{[k_1,k_2]} (v_1) - X_G^{[k_1,k_2]} (v_2)\|_p - \| X_{G'}^{[k_1',k_2']}(w_1) - X_{G'}^{[k_1',k_2']}(w_2)\|_p \biggr|^p \biggr)^{\frac{1}{p}} \\
    < &\left( (|V|-1) (\|L_G\|_1 + \|L_{G'}\|_1)^p
    (C_1 \mu_{2,G}^{k_2} + 2\|\pi_G\|_p\|Z_G\|_p\|\epsilon\|_p + C_2 \mu_{2,G'}^{k_2'})^p  \right)^{\frac{1}{p}} \\
    < &|V|^{\frac{1}{p}} (\|L_G\|_1 + \|L_{G'}\|_1)
    (C_1 \mu_{2,G}^{k_2} + 2\|\pi_G\|_p\|Z_G\|_p\|\epsilon\|_p + C_2 \mu_{2,G'}^{k_2'}).
\end{split}
\end{align}
where the abbreviation $X_G^{[k_1,k_2]}(v)$ denotes the label of the node $v$ obtained from the matrix $(M_G^{k_1} X)^T M_G^{k_2}$. If we further assume that $|E| = |E(G)| = |E(G')|$, then
\begin{align}
\begin{split}
    &\|\varphi_{H_1}^{[k_1, k_2]}(G) - \varphi_{H_1}^{[k_1', k_2']}(G')\|_p  \\
    \leq &\biggl( \sum_{n=1}^{|E| - |V| + 1} \max_{\substack{v_1, v_2 \in V(G) \\ w_1, w_2 \in V(G')}} \biggl| \|X_G^{[k_1,k_2]} (v_1) - X_G^{[k_1,k_2]} (v_2)\|_p - \| X_{G'}^{[k_1',k_2']}(w_1) - X_{G'}^{[k_1',k_2']}(w_2)\|_p \biggr|^p \biggr)^{\frac{1}{p}} \\
    < &\left( (|E| - |V| + 1) (\|L_G\|_1 + \|L_{G'}\|_1)^p
    (C_1 \mu_{2,G}^{k_2} + 2\|\pi_G\|_p\|Z_G\|_p\|\epsilon\|_p + C_2 \mu_{2,G'}^{k_2'})^p  \right)^{\frac{1}{p}} \\
    < &(|E| - |V| + 1)^{\frac{1}{p}} (\|L_G\|_1 + \|L_{G'}\|_1)
    (C_1 \mu_{2,G}^{k_2} + 2\|\pi_G\|_p\|Z_G\|_p\|\epsilon\|_p + C_2 \mu_{2,G'}^{k_2'}).
\end{split}   
\end{align}
\end{proof}

The stability theorem for reduced representations obtained from (\ref{eq:second_representation}) can be stated for graphs whose maximum weighted node degrees are sufficiently smaller than the sums of all weighted node degrees.
\begin{corollary} [Stability for reduced representations]
\label{corollary:appendix_stability}
Let $G,G'$ be two connected graphs with the same number of nodes $V = V(G) = V(G')$. Assume the conditions on random walks over $G$ and $G'$ stated in Theorem \ref{theorem:stability} hold. Denote by $d_{max}$ the maximum weighted node degrees of both $G$ and $G'$. Denote by $d_G$ and $d'_G$ the sums of all weighted node degrees of $G$ and $G'$, respectively. Suppose that no bias terms are uniformly added to the heights over edges, i.e. $\tau = 0$. Then there exist constants $C_1$ and $C_2$
\begin{tiny}
\begin{equation}
    \|\varphi_{H_0,Opt}^{[k_1,k_2]}(G) - \varphi_{H_0,Opt}^{[k_1',k_2']}(G')\|_1 < |V| \left( C_1 \mu_{2,G}^{k_2} + 2 \|\pi_G\|_1 \|Z_G\|_1 \|\epsilon\|_1 + C_2 \mu_{2,G'}^{k_2'} + \max \left\{ \frac{d_{max}}{d_G}, \frac{d_{max}}{d_{G'}} \right\} \right)
\end{equation}
\end{tiny}
Furthermore, if the number of edges of $G$ and $G'$ are equal, then
\begin{tiny}
\begin{equation}
    \|\varphi_{H_1,Opt}^{[k_1,k_2]}(G) - \varphi_{H_1,Opt}^{[k_1',k_2']}(G')\|_1 < (|E| - |V| + 1) \left( C_1 \mu_{2,G}^{k_2} + 2 \|\pi_G\|_1 \|Z_G\|_1 \|\epsilon\|_1 + C_2 \mu_{2,G'}^{k_2'} + \max \left\{ \frac{d_{max}}{d_G}, \frac{d_{max}}{d_{G'}} \right\} \right)
\end{equation}
\end{tiny}
\end{corollary}
\begin{proof}
Denote by $\overline{D}_E$ and $\overline{D}_{E'}$ the sets of tuples of unweighted degrees of two nodes of $G$ and $G'$ connected by edges in $E$ and $E'$. The union $\overline{D} := \overline{D}_{E} \cup \overline{D}_{E'}$ is of size at most $\frac{d_{max}^2}{2}$. Given a tuple of unweighted node degrees $\overline{d}$, we denote by $\overline{D}_E(\overline{d})$ the set of edges in $G$ whose associated tuple is equal to $\overline{d}$:
\begin{equation}
    \overline{D}_E(\overline{d}) := \{ e \in E(G) \; | \; \overline{d_e} = \overline{d} \}
\end{equation}
Given a choice of height functions over edges $h_E$, one may decompose the set $\overline{D}_E(\overline{d})$ into two subsets, based on whether the addition of an edge $e$ in the nested sequence of subgraphs induced from $h_E$ lead to variations in the number of connected components or the number of cycles:
\begin{align}
    \overline{D}_E(\overline{d}) &= \overline{D}_{E,0}(\overline{d}) \cup \overline{D}_{E,1}(\overline{d}) \\
    \overline{D}_{E,0}(\overline{d}) &:= \{ e_i \in \tilde{E}(G) \; | \; d_{e_i} = \overline{d}, \; h_0^i > 0 \} \\
    \overline{D}_{E,1}(\overline{d}) &:= \{ e_i \in \tilde{E}(G) \; | \; d_{e_i} = \overline{d}, \; h_1^i > 0 \}
\end{align}
Here, $\tilde{E}(G)$ is the sorted set of edges of $G$ based on the predetermined height function $h_E$ over $E(G)$.

To make the proof of the corollary simpler, we assume that for all tuples of unweighted node degrees $\overline{d} \in \overline{D}$ and for all $j = 0,1$, 
\begin{equation}
    |\overline{D}_{E,j}(\overline{d})| \geq |\overline{D}_{E',j}(\overline{d})|.
\end{equation}

We observe that the list of heights over edges utilized for constructing the representations $\varphi_{H_0,Opt}$ and $\varphi_{H_1,Opt}$ are stable with respect to graph perturbations, as shown in Theorem \ref{theorem:stability}. On the other hand, these reduced vectors are constructed from taking summations of heights over edges $e$ with a predetermined associated tuple of unweighted node degrees $\overline{d_e} = (d_{v_1},d_{v_2})$.

We use the abbreviation $X_G^{[k_1,k_2]}(v)$ to denote the label of the node $v$ obtained from the matrix $(M_G^{k_1} X)^T M_G^{k_2}$. As stated in the statement of the theorem, we ignore the contributions to the $L^1$-distance originating from the bias terms $\tau$. Then the $L^1$-distance between $\varphi_{H_0,Opt}^{[k_1,k_2]}(G)$ and $\varphi_{H_0,Opt}^{[k_1',k_2']}(G')$ is given by
\begin{small}
\begin{align}
\begin{split}
    & \|\varphi_{H_0,Opt}^{[k_1,k_2]}(G) - \varphi_{H_0,Opt}^{[k_1',k_2']}(G')\|_1 \\
    \leq & \sum_{\overline{d} \in \overline{D}} \; \biggl| \sum_{(v_1,v_2) \in \overline{D}_{E,0}(\overline{d})} \|X_G^{[k_1,k_2]}(v_1) - X_G^{[k_1,k_2]}(v_2)\|_1 - \sum_{(w_1,w_2) \in \overline{D}_{E',0}(\overline{d})} \|X_{G'}^{[k_1',k_2']}(w_1) - X_{G'}^{[k_1',k_2']}(w_2)\|_1 \biggr| \\
    \leq & \sum_{\overline{d} \in \overline{D}} \biggl( |\overline{D}_{E',0}(\overline{d})| \times \max_{\substack{v_1,v_2 \in V(G) \\ w_1,w_2 \in V(G')}} \biggl\| \|X_G^{[k_1,k_2]} (v_1) - X_G^{[k_1,k_2]} (v_2)\|_1 - \| X_{G'}^{[k_1',k_2']}(w_1) - X_{G'}^{[k_1',k_2']}(w_2)\|_1 \biggr| \\
    & + (|\overline{D}_{E,0}(\overline{d})| - |\overline{D}_{E',0}(\overline{d})|) \times \max_{v_1,v_2 \in V(G)} \biggl\| X_G^{[k_1,k_2]} (v_1) - X_G^{[k_1,k_2]} (v_2) \biggr\|_1 \biggr)
\end{split}
\end{align}
\end{small}
By perturbation theory of Markov chains and the stationary distribution of random walks over graphs, the inequality simplifies to:
\begin{small}
\begin{align}
    \begin{split}
    < & \sum_{\overline{d} \in \overline{D}} \biggl( |\overline{D}_{E',0}(\overline{d})| \times (C_1 \mu_{2,G}^{k_2} + 2 \|\pi_G\|_1 \|Z_G\|_1 \|\epsilon\|_1 + C_2 \mu_{2,G'}^{k_2'}) \\
    & + (|\overline{D}_{E,0}(\overline{d})|) \times \max \left\{ \frac{d_{max}}{\sum_{v \in V(G)} d_v}, \frac{d_{max}}{\sum_{w \in V(G')} d_w} \right\} \biggr) \\
    \leq & |V| \left( C_1 \mu_{2,G}^{k_2} + 2 \|\pi_G\|_1 \|Z_G\|_1 \|\epsilon\|_1 + C_2 \mu_{2,G'}^{k_2'} + \max \left\{ \frac{d_{max}}{\sum_{v \in V(G)} d_v}, \frac{d_{max}}{\sum_{w \in V(G')} d_w} \right\} \right).
    \end{split}
\end{align}
\end{small}
Likewise, if the number of edges of $G$ and $G'$ are equal, then the $L^1$-distance between $\varphi_{H_1,Opt}^{[k_1,k_2]}(G)$ and $\varphi_{H_1,Opt}^{[k_1',k_2']}(G')$ is given by
\begin{scriptsize}
\begin{align}
\begin{split}
    & \|\varphi_{H_0,Opt}^{[k_1,k_2]}(G) - \varphi_{H_0,Opt}^{[k_1',k_2']}(G')\|_1 \\
    < & (|E|-|V|+1) \left( C_1 \mu_{2,G}^{k_2} + 2 \|\pi_G\|_1 \|Z_G\|_1 \|\epsilon\|_1 + C_2 \mu_{2,G'}^{k_2'} + \max \left\{ \frac{d_{max}}{\sum_{v \in V(G)} d_v}, \frac{d_{max}}{\sum_{w \in V(G')} d_w} \right\} \right).
\end{split}
\end{align}
\end{scriptsize}
\end{proof}

\subsection{Time Complexity}
\label{subsection:time_complexity}

The computational complexity for calculating the matrix of node labels $(M_G^{k_1}X)^T M_G^{k_2}$ is of order $\mathcal{O}((k_1 + k_2) \times m \times l)$, for there are a total of $m$ non-zero entries in $M_G$. As for the task of sorting all the weights on edges, one needs to sort $m$ edges to formulate a nested sequence of subgraphs of $G$ given each pair of numbers $k_1$ and $k_2$, requiring $\mathcal{O}(m \log m)$ computational complexity. In addition, for each pair $(k_1,k_2)$ the computational complexity for formulating persistent vectors from a set of sorted edges is given by $\mathcal{O}(m \times \alpha(m))$, as shown in \cite{RBB19}. Here, $\alpha(m)$ is a slowly increasing function of $m$ which satisfies $\alpha(m) \ll \log m$. This leads us to conclude that the total time complexity for embedding graphs using the PWLR procedure up to $k_1$-iterations of WL kernel and $k_2$-iterations of Markov chain is $\mathcal{O}(k_1 \times k_2 \times m \times (l + \log m))$. 

\end{document}